\PassOptionsToPackage{table}{xcolor}

\documentclass[pdflatex,iicol]{sn-jnl} % Springer Nature Journal class

\usepackage{graphicx}%
\usepackage{multirow}%
\usepackage{amsmath,amssymb,amsfonts}%
\usepackage{amsthm}%
\usepackage{mathrsfs}%
\usepackage[title]{appendix}%
\usepackage{textcomp}%
\usepackage{manyfoot}%
\usepackage{booktabs}%
\usepackage{algorithm}%
\usepackage{algorithmicx}%
\usepackage{algpseudocode}%
\usepackage{listings}%
\usepackage{caption}%
\usepackage{xurl} % safe to add after xcolor option passed
\PassOptionsToPackage{table,dvipsnames,svgnames}{xcolor}

\usepackage{graphicx}
\usepackage{epsfig}
\usepackage{cleveref}
\usepackage{cite}
\usepackage{xcolor} % Now only one version loaded with all options
\usepackage{xurl}
\usepackage{hyperref}

\hypersetup{
  colorlinks,
  linkcolor={red!50!black},
  citecolor={blue!50!black},
  urlcolor={blue!80!black}
}

\usepackage{bm}
\usepackage{breqn}
\usepackage{amssymb}
\usepackage{tabularx}
\usepackage{booktabs}
\usepackage{dcolumn}
\newcolumntype{d}[1]{D..{#1}}
\newcolumntype{C}{>{\centering\arraybackslash}X}

\usepackage{soul}
\usepackage{amsthm}
\usepackage{array}
\usepackage{graphicx}  % redundant but safe

\theoremstyle{plain}

\newtheorem{thm}{Theorem}
\newtheorem{lemma}{Lemma}
\newtheorem{definition}{Definition}

\theoremstyle{definition}

\newtheorem{prob}{Problem}
\newtheorem{assumption}{Assumption}
\newtheorem{remark}{Remark}

\newcommand{\reals}{\mathbb{R}}

\newcommand{\R}{\reals}
\newcommand{\Rnonneg}{\reals_{\geq 0}}
\newcommand{\Rplus}{\reals_{>0}}

\newcommand{\nonnegintegers}{\mathbb{Z}_0}
\newcommand{\posintegers}{\mathbb{Z}_+}
\newcommand{\naturals}{\mathbb{N}}

\newcommand{\Acal}{\mathcal{A}}
\newcommand{\Bcal}{\mathcal{B}}

\newcommand{\Ncal}{\mathcal{N}}

\newcommand{\Pcal}{\mathcal{P}}
\newcommand{\Rcal}{\mathcal{R}}

\newcommand{\Tcal}{\mathcal{T}}
\newcommand{\Ucal}{\mathcal{U}}
\newcommand{\Xcal}{\mathcal{X}}

\newcommand{\Zcal}{\mathcal{Z}}

\newcommand{\eqn}[1]{\begin{align} #1 \end{align}}
\newcommand{\eqnN}[1]{\begin{align*} #1 \end{align*}}

\newcommand{\norm}[1]{\left\Vert #1 \right \Vert}

\definecolor{mygreen}{RGB}{200, 255, 200}
\definecolor{myred}{RGB}{255, 200, 200}
\definecolor{mypink}{RGB}{255, 220, 220}
\definecolor{yellow}{cmyk}{0.0,0.10,0.95,0.0}
\definecolor{pred}{cmyk}{0,0.8,0.70,0.0}
\definecolor{bluedefined}{cmyk}{0.46, 0.10, 0, 0.0}

\newcommand{\genTISD}{\textbf{\texttt{genTISD}}}

\newcommand{\mEclares}{\textbf{\texttt{mEclares}}}

\newcommand{\eware}{\textbf{\texttt{eware}}}
\newcommand{\mesch}{\textbf{\texttt{meSch}}}
\newcommand{\Rmesch}{\textbf{\texttt{RmeSch}}}
\newcommand{\Robustmesch}{\textbf{\texttt{Robust-meSch}}}

\newcommand{\gware}{\textbf{\texttt{gware}}}

\usepackage[natbibapa]{apacite}

\algrenewcommand\textproc{}

\usepackage[font=small,skip=10pt]{caption}
\algrenewcommand\algorithmicrequire{\textbf{Input:}}
\algrenewcommand\algorithmicensure{\textbf{Output:}}

\hypersetup{
  colorlinks,
  citecolor=teal,
  linkcolor=teal,
  urlcolor=purple
}

\begin{document}

\title[Article Title]{Adaptive Ergodic Search with Energy-Aware Scheduling for Persistent Multi-Robot Missions}

\author*[1]{\fnm{Kaleb Ben} \sur{Naveed}}\email{kbnaveed@umich.edu}
\author[1]{\fnm{Devansh R.} \sur{Agrawal}}\email{devansh@umich.edu}
\author[1]{\fnm{Rahul} \sur{Kumar}}\email{rahulhk@umich.edu}
\author[1,2]{\fnm{Dimitra} \sur{Panagou}}\email{dpanagou@umich.edu}

\affil*[1]{\orgdiv{Department of Robotics}, \orgname{University of Michigan}, \city{Ann Arbor}, \postcode{48109}, \state{MI}, \country{USA}}
\affil[2]{\orgdiv{Department of Aerospace Engineering}, \orgname{University of Michigan}, \city{Ann Arbor}, \postcode{48109}, \state{MI}, \country{USA}}

\abstract{
Autonomous robots are increasingly deployed for long-term information-gathering tasks, which pose two key challenges: planning informative trajectories in environments that evolve across space and time, and ensuring persistent operation under energy constraints. This paper presents a unified framework, \mEclares{}, that addresses both challenges through adaptive ergodic search and energy-aware scheduling in multi-robot systems. Our contributions are two-fold: (1) we model real-world variability using stochastic spatiotemporal environments, where the underlying information evolves unpredictably due to process uncertainty. To guide exploration, we construct a target information spatial distribution (TISD) based on clarity, a metric that captures the decay of information in the absence of observations and highlights regions of high uncertainty; and (2) we introduce \Robustmesch{} (\Rmesch{}), an online scheduling method that enables persistent operation by coordinating rechargeable robots sharing a single mobile charging station. Unlike prior work, our approach avoids reliance on preplanned schedules, static or dedicated charging stations, and simplified robot dynamics. Instead, the scheduler supports general nonlinear models, accounts for uncertainty in the estimated position of the charging station, and handles central node failures. The proposed framework is validated through real-world hardware experiments, and feasibility guarantees are provided under specific assumptions.\\
\href{https://github.com/kalebbennaveed/mEclares-main.git}{[Code: https://github.com/kalebbennaveed/mEclares-main.git]}\\
\href{https://www.youtube.com/watch?v=dmaZDvxJgF8}{[Experiment Video: https://www.youtube.com/watch?v=dmaZDvxJgF8]}
}


\maketitle

\section{Introduction}\label{sec1}
Autonomous robots are increasingly deployed in missions requiring long-term data acquisition, such as environmental monitoring~(\cite{water_exploration1, water_exploration2}), ocean current characterization~(\cite{gulf1, gulf2}), wildfire surveillance~(\cite{julian2019distributed}), and search-and-rescue operations~(\cite{search_rescue1, search_rescue2}). Planning informative trajectories for such missions poses two key challenges: (i) designing robot trajectories that maximize information acquisition in dynamic environments, and (ii) ensuring task persistence under energy constraints by enabling timely recharging.

\subsection{Adaptive Informative Path Planning}

The first challenge involves adaptive planning in spatiotemporal environments—where quantities of interest (e.g., temperature, wind speed, gas concentration) evolve across space and time. Informative path planning (IPP) addresses this by generating robot paths that maximize information gain or minimize uncertainty, subject to resource constraints. Classical approaches include orienteering-based formulations~(\cite{bottarelli2019orienteering}), submodular optimization methods~(\cite{meliou2007nonmyopic}), and Gaussian Process (GP)-based planners~(\cite{Chen-RSS-22}). To improve adaptability and scalability in high-dimensional settings, sampling-based methods~(\cite{moon2025ia}) and receding horizon strategies~(\cite{sun2017no}) have been proposed. However, many struggle to adapt in real-time to variations in the environment.

Ergodic search offers an alternative by generating trajectories that match the time-averaged visitation frequency with a target information spatial distribution (TISD), instead of choosing discrete sensing points. Prior work~(\cite{Mezic_Ergodic, Dressel_Ergodic, abraham2021ergodic, cmu_ergodic, dong2023time}) has demonstrated its value in achieving spatially balanced exploration. However, these methods often assume \textit{spatiostatic environments}(\cite{Mezic_Ergodic, dong2023time}) or rely on known spatiotemporal dynamics(\cite{Dressel_Ergodic, multiobjectiveergodic, Candela_thesis}), limiting applicability in real-world scenarios where uncertainty arises from model mismatch, disturbances, or environmental variability.

To address this, we consider \textit{stochastic spatiotemporal environments}—environments whose evolution is uncertain in both space and time. In such cases, information can decay without continued measurement, motivating online trajectory planning that prioritizes regions with high uncertainty and rapid information loss. We build on the clarity metric, proposed by \cite{clarity}, a bounded information measure between $[0,1]$ that captures both current knowledge and its decay due to lack of observation. Using clarity, we construct a principled TISD that continuously evolves based on the robot's measurement history and environmental uncertainty, allowing robots to adaptively revisit regions where uncertainty is increasing.

\subsection{Task Persistence in Multi-Agent Systems}
\label{sec:related_work_task_persistence}

The second challenge is persistent operation under energy constraints, particularly when multiple robots must coordinate recharging through a shared charging resource. Prior work on task persistence spans both single-agent and multi-agent scenarios involving static and mobile charging infrastructure.

For static stations, some methods assume a dedicated charger per robot~(\cite{dedicated_01, dedicated_02, dedicated_03, multirobot_game_theoretic}), while others support shared chargers with concurrent access~(\cite{singall_for_all_1, periodic_charging}). When fewer chargers than robots are available~(\cite{static_charging_bipartite, static_placing_charging, seewald2024energyaware}), strategies include modifying mission paths~(\cite{static_charging_bipartite}), placing stations strategically~(\cite{static_placing_charging}), or constraining charging frequency~(\cite{seewald2024energyaware}). Closest to our work are~\cite{bentz2018complete, persis_Fouad}: the former staggers robot deployments to ensure exclusivity, while the latter employs control barrier functions (CBFs)~\cite{CBF_TAC} to enforce minimum SoC levels under simplified single-integrator dynamics.

Most mobile charging approaches assume a dedicated charging robot, with coordination either via precomputed rendezvous points~(\cite{pre-plan_1, pre-plan_2}) or continuous communication~(\cite{continous_comms_1}). Others dynamically intercept robots during their mission~(\cite{cooperative_1, cooperative_2, MOBILE_RAL_2022}). In contrast, we consider a shared mobile charging station that travels alongside the robot network to extend operational time. Our method does not rely on preplanned rendezvous or continuous communication and supports general nonlinear robot dynamics.

\subsection{Contributions}

This work presents a unified framework for adaptive ergodic search and energy-aware scheduling in persistent multi-robot missions. Our key contributions, situated in the context of existing state-of-the-art methods, are:

\begin{itemize}
\item \textbf{Principled multi-agent TISD construction via clarity:}
Unlike prior ergodic methods that assume static~(\cite{ dong2023time}) or known spatiotemporal dynamics~(\cite{Dressel_Ergodic, multiobjectiveergodic, Candela_thesis}), we construct the target information spatial distribution (TISD) using the clarity metric~(\cite{clarity}), a bounded measure that quantifies information decay and the maximum attainable information in stochastic spatiotemporal environments. This allows robots to adaptively focus sensing effort in regions with high uncertainty and rapid information loss. 

\item \textbf{Robust energy-aware scheduling with fail-safe coordination:}  
Unlike prior work that achieves exclusivity through staggered deployment~(\cite{bentz2018complete}) or relies on simplified single-integrator dynamics with fixed SoC thresholds~(\cite{persis_Fouad}), we propose \Robustmesch{} (\Rmesch{}), a centralized online scheduling framework that supports general nonlinear robot dynamics, enforces exclusive access to a shared mobile charging station, and guarantees safe returns through a decentralized fail-safe planner that accounts for communication delays and central node failures. Furthermore, we provide formal feasibility guarantees and derive conditions under which robots can be safely added to or removed from the mission without violating energy and return-gap constraints.

\item \textbf{Hardware-validated multi-agent coordination:}  
We validate the proposed method on a heterogeneous team comprising multiple aerial robots and a mobile ground-based charging station through extensive hardware experiments. 
\end{itemize}

\textbf{\textit{Comparison to our own earlier works:}} Compared to our earlier conference papers, this work introduces several key extensions:
\begin{itemize}
\item Compared to~\cite{naveed2024eclares}, we extend the clarity-based information model to the multi-agent case, enabling distributed sensing and coordination.
\item Compared to~\cite{naveed2024mesch}, we introduce a fail-safe planner that enables safe recovery under central node failures and provide a more comprehensive theoretical analysis, including formal guarantees on feasibility and robustness.
\item In addition, this paper presents an expanded experimental evaluation compared to both prior works, including real-world demonstrations involving multiple aerial robots coordinating through a shared mobile charging station.
\end{itemize}

\section{Preliminaries}
\subsection{Notation}
Let $\nonnegintegers = \{ 0, 1, 2, ... \}$ and $\posintegers = \{1, 2, 3, ... \}$. Let $\mathbb{R}$, $\mathbb{R}_{\geq 0}$, $\mathbb{R}_{> 0}$ be the set of reals, non-negative reals, and positive reals respectively. Let $\mathbb{S}^{n}_{++}$ denote set of symmetric positive-definite matrices in $\mathbb{R}^{n \times n}$. Let $\mathcal{N}(\mu, \Sigma)$ denote a normal distribution with mean $\mu$ and covariance $\Sigma \in \mathbb{S}^{n}_{++}$. The $Q \in \mathbb{S}^{n}_{++}$, norm of a vector $x \in \R^n$ is denoted $\norm{x}_Q = \sqrt{x^T Q x}.$ The space of continuous functions $f : \Acal \rightarrow \Bcal$ is denoted as $C(\Acal, \Bcal)$. 

\subsection{System Description}
Consider a multi-agent system, in which each robotic system $i \in \Rcal = \{1, \cdots, N \}$, referred to as a \textbf{\textit{rechargeable robot}}, comprises the robot and battery discharge dynamics: 
\begin{equation}
\label{eqn:rechargeable}
     \Dot{\chi}^i =
    \begin{bmatrix}
        \Dot{x}^i \\
        \Dot{e}^i 
    \end{bmatrix}
    =
    f^i(\chi^i, u^i) \\
    =
    \begin{bmatrix}
        f_r^i(x^i, u^i) \\
        f_e^i(e^i)
    \end{bmatrix},
\end{equation}
where $N = |\Rcal|$ is the cardinality of the set $\Rcal$, $\chi^i = \begin{bmatrix}{x^i}^T, & e^i\end{bmatrix}^T \in \Zcal^i_r \subset \R^{n+1}$ is the $i^{th}$ robotic system state consisting of the robot state $x^i \in \mathcal{X}^i_r \subset \mathbb{R}^{n}$ and its State-of-Charge (SoC) $e^i \in \mathbb{R}_{\geq 0}$. $u^i \in \Ucal^i_r \subset \mathbb{R}^{m}$ is the control input, $f^i: \Zcal^i_r \times \Ucal^i_r \rightarrow \mathbb{R}^{n+1}$ defines the continuous-time robotic system dynamics, $f_r^i: \Xcal^i_r \times \Ucal^i_r \rightarrow \mathbb{R}^n$ define robot dynamics and $f_e^i: \mathbb{R}_{\geq 0} \to  \mathbb{R}$ define worst-case battery discharge dynamics. We also consider the continuous-time dynamics of the mobile charging station (referred to as \textbf{\textit{mobile charging robot}}):
\begin{subequations}
\label{eqn:mobile_charging}
\eqn{
    \Dot{x}^c &= f_c(x^c, u^c) + w(t),   &&w(t) \sim \mathcal{N} (0, W(t)), \label{eqn: charge_dynamics} \\
    y^c &= z(x^c) + v(t), &&v(t) \sim \mathcal{N} (0, V(t)), \label{eqn:obs_nl_model}
}    
\end{subequations}
where $x^c \in \mathcal{X}_c \subset \mathbb{R}^{c}$ is the charging station state, $u^c \in \mathcal{U}_c \subset \mathbb{R}^{s}$ is the charging station control input, $f_c: \Xcal_c \times \Ucal_c \rightarrow \mathbb{R}^{c}$ defines the continuous-time system dynamics for the mobile charging, $w(t)$ is the time-varying process noise with zero mean and known variance $W(t) \in \mathbb{R}_{\geq 0} $, $y^c \in \mathbb{R}^{c}$ is the measurement, $z : \mathbb{R}^{c} \to \mathbb{R}^{c}$ is the observation model, and $v(t)$ is the time-varying measurement noise with zero mean, and known covariance $V(t)$.

\subsection{Ergodic Search}\label{sec:prelim_ergodic}

Ergodic search~(\cite{Mezic_Ergodic, Dressel_Ergodic}) is a technique to generate trajectories $x: [t_0, T] \to \mathcal{X}$ that cover a rectangular domain  $\mathcal{P} = [0, L_1] \times \cdots [0, L_s] \subset \mathbb{R}^s$, matching a specified \emph{target information spatial distribution} (TISD) $\phi : \Pcal \to \R$, where $s$ is the dimensionality of the environment and $\phi(p)$ is the density at $p \in \mathcal{P}$. 
Moreover, the spatial distribution of the trajectory $x(t)$ is defined as
\begin{equation}
    c(x(t), p) = \frac{1}{T-t_0}\int_{t_0}^{T} \delta(p - \Psi(x(\tau))) d\tau
\end{equation}
where $\delta: \Pcal \to \R$ is the Dirac delta function and $\Psi: \Xcal \to \Pcal$ is a mapping such that $\Psi(x(\tau))$ is the position of the robot at time $\tau \in [t_0, T]$. In other words, given a trajectory $x(t)$, $c(x(t), p)$ represents the fraction of time the robot spends at a point $p \in \Pcal$ over the interval $[t_0, T]$. Then, the \emph{ergodicity} of $x(t)$ w.r.t to a TISD $\phi$ is
\eqn{
\Phi(x(t), \phi) = \norm{ c - \phi}_{H^{-(s+1)/2}}
}
where $\norm{\cdot}_{H^{-(s+1)/2}}$ is the Sobolev space norm defined in~\cite{Mezic_Ergodic}, i.e., $\Phi$ is a function space norm measuring the difference between the TISD $\phi$ and the spatial distribution of the trajectory $c$. Given the ergodic metric, ergodic trajectories for a team of $N$  robots can be computed by solving the following optimization problem over the space of trajectories  $x^i(t) \in C([t_0, T], \Xcal)$ and control inputs  $u^i(t) \in C([t_0, T], \Ucal)$ for each robot  $i \in \Rcal$:
\begin{equation}
    \label{eq:multiagent_ergodic_optimization}
    \begin{aligned}
    \min_{\{x^i(t), u^i(t)\}} \quad & \Phi(\{x^1(t), \cdots, x^N(t)\}; \phi)  \\ 
        &+ \sum_{i=1}^N \int_{t_0}^{T} \norm{u^i(\tau)}^2 d\tau \\
    \textrm{s.t.} \quad & \dot{x}^i = f_r(x^i, u^i), \quad \forall i \in \Rcal \\
    & x^i(t_0) = x^i_0 \\
    & \norm{x^i(t) - x^j(t)} \geq d_{\min}, \quad \forall i \neq j
    \end{aligned}
\end{equation}
where $x^i_0$ is the initial state of robot $i$, and $d_{\min}$ is the minimum safety distance to ensure inter-robot collision avoidance. The multi-agent ergodic metric $\Phi(\{x^1(t), \cdots, x^N(t)\}; \phi)$ quantifies the team's collective coverage of the target distribution $\phi$. It is typically computed via a Fourier decomposition of both the empirical visitation statistics and the target distribution~(\cite{Dressel_Ergodic}). This optimization problem can be solved using gradient-based methods. In this work, we do not focus on a specific trajectory optimization method, but rather on the principled construction of the TISD for guiding ergodic exploration in stochastic spatiotemporal environments.

%%% single agent trajectory optimization
% Given the ergodic metric, an ergodic trajectory can be generated by solving the following optimization problem over the space of trajectories $x(t) \in C([t_0, T], \Xcal)$ and control inputs $u(t) \in C([t_0, T], \Ucal)$:
% \begin{equation}
%     \label{eq:ergodic_optimization_cont}
%     \begin{aligned}
%     \min_{x(t), u(t)} \quad & \left[ \Phi(x(t), \phi) + \int_{t_0}^{T} \norm{u(\tau)}^2 d\tau \right] \\
%     \textrm{s.t.} \quad & \dot{x} = f_r(x, u) \\
%     & x(t_0) = x_0
%     \end{aligned}
% \end{equation}
% where $x_0$ is the initial state of the robot at time $t_0$. This problem can be solved using a gradient descent approach, where a Fourier decomposition of both $c$ and $\phi$ is used to numerically evaluate $\Phi$ and its gradients with respect to $x(t)$ and $u(t)$~\cite{Dressel_Ergodic}. In this work, our focus is not any particular method od generating trajectories; however, we focus on designing a principled TISD for stochastic spatiotemporal environments.

\subsection{Clarity}

We use Clarity~\cite{clarity}, an information measure that defines the quality of information about the variable of interest on a $[0,1]$ scale. Let $X$ be an $n$-dimensional continuous random variable with a density function $\rho(x)$. Its differential entropy is given as follows:
\eqn{
h[X] = -\int_S \rho(x)\log \rho(x) dx
}
where $S$ is the support of $X$. Clarity of $X$, derived from differential entropy, is defined as follows:
\begin{definition} The Clarity $q[X] \in [0,1]$ is defined as:
\eqn{
q[X] = \biggl( 1 + \frac{e^{2h[X]}}{(2\pi e)^n} \biggr)^{-1}
}
\end{definition}

$q \rightarrow 1$ represents the case when $X$ is perfectly known, whereas lower values correspond to higher uncertainty.

Consider a stochastic variable (quantity of interest) $m\in \R$ governed by the process and output (measurement) models:
\begin{subequations}
\eqn{
    \Dot{m} &= w(t),   &&w(t) \sim \Ncal (0, Q) \label{eqn: quantity_of_interest_point}\\ 
    y &= C(x)h+ v(t),  &&v(t) \sim \Ncal (0, R)   \label{eqn: quantity_of_interest_measurement_point}
}
\end{subequations}
where $Q \in \Rnonneg$ is the known variance associated with the process noise, $y \in \R$ is the measurement, $C: \Xcal \to \R$ is the mapping between robot state and sensor state, and $R \in \R$ is the known variance of the measurement noise.

Clarity $q$ of the random quantity $m$, which lies between $[0,1]$ and is defined such that $q = 0$ represents $m$ being unknown, and $q = 1$ corresponds to $m$ being completely known. The clarity dynamics for the subsystem \eqref{eqn: quantity_of_interest_point}, \eqref{eqn: quantity_of_interest_measurement_point} are given as follows
\begin{equation}
    \label{eqn:clarity_dynamics}
    \Dot{q} = \frac{C(x)^2}{R}(1 - q)^2 - Qq^2
\end{equation}

\section{Problem Formulation}
In this section, we provide the mathematical formulation of the problem. We first derive the clarity dynamics for multi-robot systems, then describe the environment model, and finally present the overall problem statement.

\subsection{Multi-robot Clarity Dynamics }

We consider the estimation of a scalar stochastic variable $m$ using $N$ robots. The system dynamics are:
\eqn{
\dot{m} = w(t),  &&w(t) \sim \Ncal(0, Q)
}
Let $X = [x^1, x^2, \dots, x^N]^T \in \R^{N \times 1}$ denote the stacked state vector of all robots. Each robot $i \in \Rcal$ measures $m$ as follows:
\eqn{
y^i = C(x^i)m + v^i(t),  &&v^i(t) \sim \Ncal(0, R(x_i))
}
Assuming the measurement noise is independent across agents, the measurements can be stacked as:
\eqn{
y(X) = C(X)m + v(X), &&v(t) \sim \Ncal(0, R(X)) 
}
where
\eqn{
C(X) = [C(x^1), C(x^2), \cdots, C(x^N)]^T \in \R^{N \times 1}
}
\eqn{
R(X) =
\begin{bmatrix}
R(x^1) & 0         & \cdots & 0 \\
0       & R(x^2)    & \cdots & 0 \\
\vdots  & \vdots    & \ddots & \vdots \\
0       & 0         & \cdots & R(x^N)
\end{bmatrix}
\in \R^{N \times N}
}
The Kalman filter equations for the scalar estimate $\mu$ and variance $P$ are:
\begin{subequations}
\eqn{
    \dot{\mu} &= P C(X)^T R(X)^{-1} (y(X) - C(X)\mu) \\
    \dot{P} &= Q - P C(X)^T R(X)^{-1} C(X) P
}
\end{subequations}
Since clarity is defined as $q = \frac{1}{1 + P}$, the clarity dynamics can be derived as follows:
\eqn{
\begin{aligned}
\dot{q} &= \frac{-\dot{P}}{(1 + P)^2} \\
&= \frac{1}{(1 + P)^2} \left( P^2 C(X)^T R(X)^{-1} C(X) - Q \right)
\end{aligned}
}
Substituting $P = \frac{1 - q}{q}$, we get
\eqn{
\label{multi_clarity_dynamics}
\begin{aligned}
\dot{q} &= (1 - q)^2 C(X)^T R(X)^{-1} C(X) - Qq^2  \\
&=  (1 - q)^2 \sum_{i \in \Rcal} \frac{C(x^i)^2}{R(x^i)} - Qq^2
\end{aligned}
}
The \eqref{multi_clarity_dynamics} define the clarity dynamics for the case when measurements from multiple robots are involved in estimating the quantity of interest. 

If $C(x^i)$ and $R(x^i)$ are constant for all $i \in \Rcal$, then the clarity dynamics \eqref{multi_clarity_dynamics} admit a closed-form solution for the initial condition $q(0) = q_0$:
\begin{equation}
    \label{eqn:multi_clarity_closed_form}
    q(t; q_0) = q_{\infty} \left(1 + \frac{2\gamma_1}{\gamma_2 + \gamma_3 e^{2kQt}} \right)
\end{equation}
where $k = \sqrt{ \frac{ \sum_{i \in \Rcal} \frac{C(x^i)^2}{R(x^i)} }{Q} }, q_{\infty} = \frac{k}{k+1}, \gamma_1 = q_{\infty} - q_0, \gamma_2 = \gamma_1(k - 1)$, and $\gamma_3 = (k - 1)q_0 - k$.

% \eqn{
% k = \sqrt{ \frac{ \sum_{i \in \Rcal} \frac{C(x_i)^2}{R(x_i)} }{Q} }, \quad
% q_{\infty} = \frac{k}{k+1}, \quad
% \gamma_1 = q_{\infty} - q_0, \quad
% \gamma_2 = \gamma_1(k - 1), \quad
% \gamma_3 = (k - 1)q_0 - k
% }

As $t \to \infty$, $q(t; q_0) \to q_\infty \leq 1$ monotonically. Thus $q_\infty$ defines the maximum attainable clarity. Equation~\eqref{eqn:multi_clarity_closed_form} can be inverted to determine the time required to increase clarity from $q_0$ to some $q_1$. This time is denoted $\Delta T: [0, 1]^2 \to \Rnonneg$:
\eqn{
\Delta T(q_0, q_1) = t \text{ s.t. } q(t, q_0) = q_1 \quad \text{for } q_1 \in [q_0, q_\infty) \label{eqn:multi_clarity_inverse}
}
For $q_1 < q_0$, we set $\Delta T(q_0, q_1) = 0$ while $\Delta T(q_0, q_1)$ is undefined for $q_1 \geq q_\infty$.

\subsection{Environment Specification}
Consider the coverage space $\Pcal$. We discretize the domain into a set of $N_p$ cells each with size $V$.\footnote{Size is length in 1D, area in 2D, and volume in 3D.} Let $m_p: [t_0, \infty) \to \mathbb{R}$ be the (time-varying) quantity of interest at each cell $p \in \Pcal_{\text{cells}} = \{1, ..., N_p\}$. We model the quantities of interest as independent stochastic processes:
\begin{subequations}
\label{eqn: quantity_of_interest}
\eqn{
    \Dot{m}_p &= w_{p}(t), &&w_{p}(t) \sim \mathcal{N} (0, Q_{p}) \label{eqn: quantity_of_interest_process}\\
    y_p &= C_{p}(X) m_{p} + v_p(t),  &&v_p(t) \sim \mathcal{N} (0, R(X)) \label{eqn: quantity_of_interest_measurement}
}
\end{subequations}
where $y_p \in \R$ is the output corresponding to cell $p$. $R(X)$ is the measurement noise, and $Q_p \in \Rplus$ is the process noise variance at each cell $p$. Since $m_p$ varies spatially and temporally under process noise $Q_p$ for each cell $p \in \Pcal_{\text{cells}}$, the environment becomes a \emph{stochastic spatiotemporal environment}.

\subsection{Problem Statement}
Consider a team of $N+1$ robots performing persistent coverage of a stochastic spatiotemporal environment~\eqref{eqn: quantity_of_interest} over a time horizon $[0, \infty)$. Among them, $N$ robots are rechargeable and require periodic recharging, while one robot serves as a mobile charging robot and does not require recharging.\footnote{This could represent, for instance, a ground vehicle with a battery that lasts several hours. A similar assumption is made in prior works \cite{cooperative_1, MOBILE_RAL_2022}.} The rechargeable robots model the mobile charging robot using~\eqref{eqn:mobile_charging}. The objectives for the rechargeable robots are twofold:

\begin{itemize}
    \item Generate nominal informative trajectories for rechargeable robots using clarity-driven ergodic search;
    \item Ensure mutually exclusive use of the mobile charging robot, which follows a nominal trajectory.
\end{itemize}

We formulate an optimization problem that captures these objectives. The objective function is designed to maximize clarity across the regions of interest, while constraints ensure that each robot’s energy level remains non-negative and that the robots exclusively share the single mobile charging station. We now define the clarity-based objective functional, along with the energy constraints and mutual exclusion constraints related to charging.
\begin{figure*}[t]
  \centering
  \includegraphics[width=2.1\columnwidth]{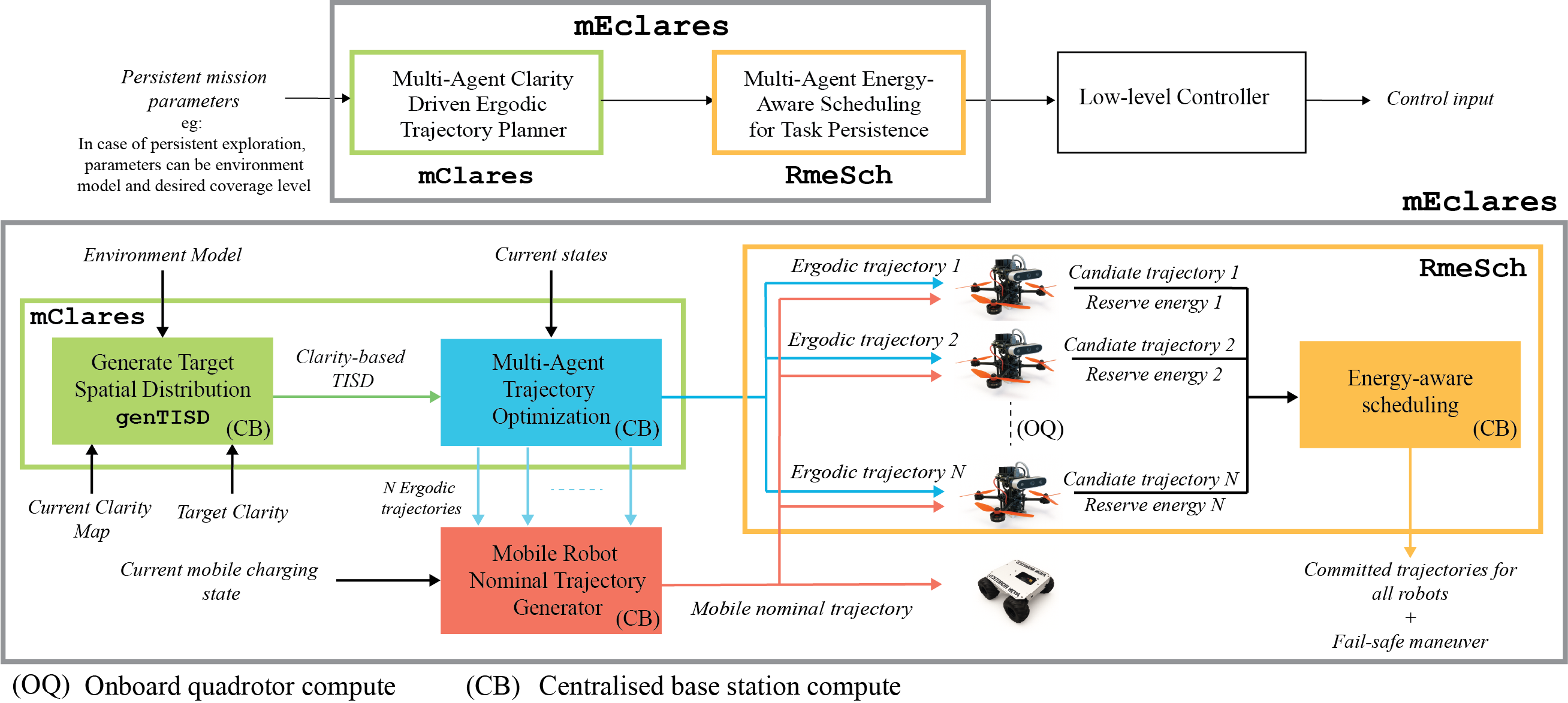}
  \caption{\mesch: The block diagram shows the complete proposed framework \mEclares. }
  \label{fig:mEclares_overview}
\end{figure*}

\begin{figure} [t]
  \centering
  \includegraphics[width=0.8\columnwidth]{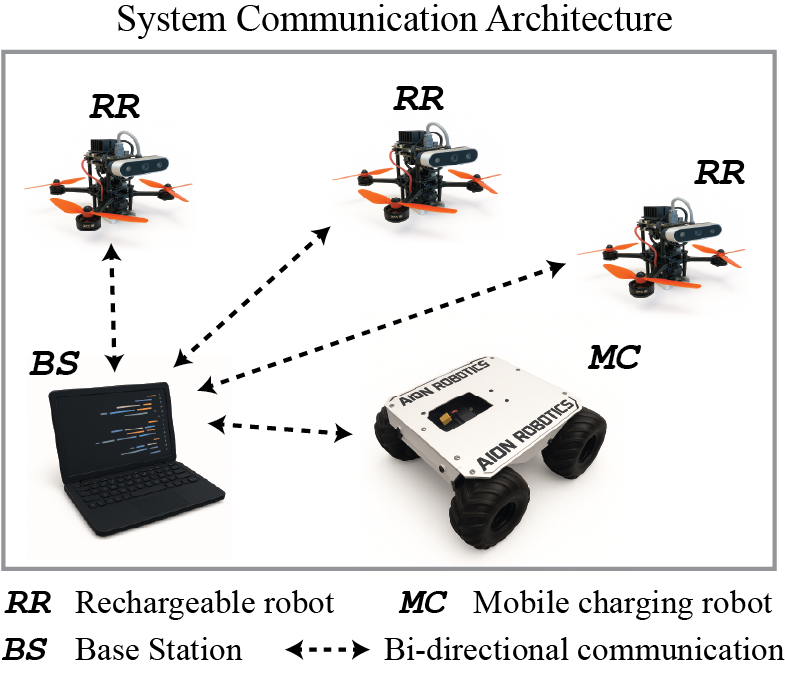}
  \caption{The supported communication architecture of the system.}
  \label{fig:comms}
\end{figure}

\subsubsection{Clarity-based Objective functional}

Assume the desired quality of information at each cell is encoded using a \emph{target clarity} $\overline q_p < q_{\infty, p}$ for each cell $p \in \Pcal_{\text{cells}}$. The target clarity can be different at each cell, indicating a different desired quality of information at each cell, but must be less than $q_{\infty, p}$, the maximum attainable clarity of the cell. If $\overline q_p \geq q_{\infty, p}$ for a cell $p \in \Pcal_{\text{cells}}$, then the robot would try to spend an infinite amount of time at a cell $p$, which is undesirable.

We use clarity as our information metric since it is particularly effective for stochastic spatiotemporal environments:
\begin{itemize}
    \item The clarity decay rate in cell $p$, i.e. $-Q_{p}q_{p}^2$, is explicitly dependent on the stochasticity of the environment $Q_{p}$ in \eqref{eqn: quantity_of_interest}. This allows the information decay rate to be determined from the environment model, and not set heuristically. Furthermore, spatiostatic environments are a special case: by setting $Q_{p} = 0$, clarity cannot decay. 
    \item While taking measurements of cell $p$, clarity $q_{p}$ monotonically approaches $q_{\infty, p} < 1$ for $Q_p, R(x^i) > 0, \forall i \in \Rcal$. This indicates that the maximum attainable information is upper bounded.
\end{itemize}

In this persistent task, the trajectory for each robot is replanned every $T_H \in \Rplus$ seconds, i.e., at times $\{t_0, t_1, \cdots \}$ for $t_k = k T_H$, $k \in \naturals$. At the $k$-th iteration, the objective is to minimize the \textit{mean clarity deficit} $q_d(t_k + T_H)$, which is defined as
\begin{equation}
    \label{eq:mean_clarity_error}
    q_d(t_{k} + T_H) =  \frac{1}{N_p}\sum_{p = 1}^{N_p}  \max (0, \overline{q}_{p} - q_{p}(t_{k} + T_H)) 
\end{equation}
where $q_{p}(t_{k} + T_H)$ is the clarity at time $t_{k} + T_H$ of cell $p \in \Pcal_{\text{cells}}$. However, in order to persistently monitor a stochastic spatiotemporal environment over a long time horizon, the robot's energy constraints must be taken into consideration.

\subsubsection{Minimum Energy and Mutually Exclusive Charging Constraints}

We define $\Tcal^i$ as the set of times $i^{th}$ robot returns to the charging station:
\eqn{
\Tcal^i = \{t_{0}^i, t_{1}^i, \cdots, t_{m}^i, \cdots\},  \forall i \in \Rcal, \forall m \in \nonnegintegers
}
where $t_{m}^i$ represents the $m^{th}$ return time of the $i^{th}$ rechargeable robot. Let $\Tcal = \cup_{i \in \Rcal} \Tcal^i$ be the union of return times for all robots. We now define two conditions that must hold for all times $t \in [t_0, \infty)$ to achieve the objectives stated above:
\begin{subequations}
\eqn{
&e^i(t) \geq {e^i_{min}} \quad  \quad \forall t \in [t_0, \infty),  \forall i \in \Rcal  \label{eq:min_energy_constraints}\\
&|t^{i_1}_{m_1} - t^{i_2}_{m_2}| > T_{\delta}  \quad  \forall t^{i_1}_{m_1},t^{i_2}_{m_2}  \in \Tcal \label{eq:gap_lower_bound}
% \probability \Bigl(\norm{x^i(t_{s,m}^i) - x^c(t_{s,m}^i)}^2_2 \leq r_{char}\Bigr) \geq \alpha \label{eq:3g_charge_final_state} + T_\delta
}
\end{subequations}
Condition \eqref{eq:min_energy_constraints}, the \textbf{\textit{minimum SoC condition}}, defines the required minimum battery SoC for all rechargeable robots. Condition \eqref{eq:gap_lower_bound}, the \textbf{\textit{minimum gap condition}}, ensures a sufficient time gap between the returns of two robots to avoid charging conflicts. The term $T_{\delta} = T_{ch} + T_{bf}$ represents the charging duration and the buffer time needed for a robot to resume its mission before the next robot arrives. 

Now we define the optimization problem, which must be solved at times $\{t_0, t_1, \cdots\}$ for $t_k = kT_H$:

\begin{prob}
At each planning time $t_k$, the problem is posed as:
\begin{subequations}    
\label{eq:overall_prob_finite_horizon}
\begin{align}
\min_{\chi^i(t),u^i(t)} \quad &  q_d(t_{k} + T_H)\\
\textrm{s.t.} \quad 
& \chi^i(t_k) = \chi^i_k, \quad \forall i \in \Rcal \\
& \dot{\chi}^i = f(\chi^i, u^i), \quad \forall i \in \Rcal  \\
& \dot{q}_{p} = g(x, q_p), \quad \forall p \in \Pcal_{\text{cells}}\\
& \norm{x^i(t) - x^j(t)} \geq d_{\min}, \forall i \neq j \label{eq:col_avoid_constraint}\\
& e^i(t) \geq e^i_{\min}, \quad \forall i \in \Rcal \label{eq:finite_energy_constraint} \\
& |t^{i_1}_{m_1} - t^{i_2}_{m_2}| > T_{\delta}, \quad \forall t^{i_1}_{m_1}, t^{i_2}_{m_2} \in \Tcal_{k,H}
\label{eq:finite_gap_constraint}
\end{align}
\end{subequations}
\end{prob}
where $q_d(t_{k} + T_H)$ is the mean clarity deficit at the end of system trajectory $\chi^i(t; t_k, \chi_k)$, $\forall t \in [t_k, t_{k} + T_H]$, $\forall i \in \Rcal$ given by \eqref{eq:mean_clarity_error}, $g : \mathcal{X} \times [0, 1] \rightarrow \mathbb{R}_{\geq 0}$ define the clarity dynamics \eqref{multi_clarity_dynamics}, and $e_{min}$ is the minimum energy level allowed for the robot. \eqref{eq:col_avoid_constraint} defines the collision avoidance constraint for all robots $i, j \in \Rcal$. The set $\Tcal_{k,H} = \Tcal \cap [t_k, t_k + T_H]$ denotes all charging return times within the current planning horizon, and is used to enforce the minimum gap condition over this finite window.

\section{Method Motivation \& Overview}

\subsection{Method Motivation}

To solve problem~\eqref{eq:overall_prob_finite_horizon}, we draw inspiration from ergodic search. As discussed in~\cref{sec:prelim_ergodic}, ergodic search generates trajectories by solving problem~\eqref{eq:multiagent_ergodic_optimization}. When the target information spatial distribution (TISD) $\phi$ is constructed based on the current clarity $q_p(t)$ and a desired target clarity $\overline{q}_p$ at each cell, ergodic search naturally minimizes the mean clarity deficit~\eqref{eq:mean_clarity_error}. In this work, we propose a principled method to construct $\phi$ using clarity.

However, the optimization in~\eqref{eq:multiagent_ergodic_optimization} does not account for energy constraints~\eqref{eq:min_energy_constraints} or the minimum gap requirement~\eqref{eq:gap_lower_bound}. While one could include these constraints in~\eqref{eq:multiagent_ergodic_optimization}, the non-convexity of the problem makes it difficult to ensure convergence or feasibility. We therefore propose \mEclares{}, shown in~\cref{fig:mEclares_overview}, as an approximate solution to~\eqref{eq:overall_prob_finite_horizon}.\footnote{Note that $q_d(T)$ is not differentiable, making direct optimization of~\eqref{eq:overall_prob_finite_horizon} challenging. In contrast,~\eqref{eq:multiagent_ergodic_optimization} is differentiable and can be efficiently approximated using gradient-based trajectory optimization solvers.}

\subsection{Method Overview}
Our approach decouples~\eqref{eq:overall_prob_finite_horizon} into two sub-problems: (A) each robot computes an \emph{ergodic trajectory} that maximizes information collection while ignoring energy constraints; (B) each robot then generates a \emph{candidate trajectory} that attempts to track a portion of the ergodic trajectory while reaching the charging station before depleting its energy. All candidate trajectories are sent to the base computer, where the \Rmesch{} algorithm evaluates them and decides whether to \emph{commit} each one. Committed trajectories are guaranteed to satisfy the minimum SoC constraint~\eqref{eq:min_energy_constraints} and the minimum gap constraint~\eqref{eq:gap_lower_bound}. Each robot always tracks its most recent committed trajectory, ensuring persistent exploration while respecting energy constraints and coordinating exclusive access to the mobile charging station. The nominal trajectory of the mobile charging robot is generated so that it travels along the network of rechargeable robots. 

These components operate on different timescales. The ergodic trajectory is replanned every $T_H$ seconds, while the committed trajectory is updated every $T_E < T_H$ seconds.\footnote{$T_E, T_H \in \mathbb{R}_+$ are user-defined parameters.}
\begin{itemize}
    \item At each time $t_k = k T_H, \ k \in \mathbb{N}$:
    \begin{itemize}
        \item Recompute the TISD $\phi$ using \genTISD{}.
        \item Recompute the ergodic trajectory for the rechargeable robots and the nominal trajectory of the mobile charging robot.
    \end{itemize}
    \item At each time $t_j = j T_E, \ j \in \mathbb{N}$:
    \begin{itemize}
        \item Each robot generates a candidate trajectory and sends it to the base computer.
        \item The central \Rmesch{} algorithm evaluates the candidate trajectories and decides whether to commit each one of them.
        \item \Rmesch{} also publishes the fail-safe schedule in case the central node fails before the next decision iteration $j+1$.
    \end{itemize}
\end{itemize}
\begin{algorithm}
\footnotesize
\caption{The \genTISD{} algorithm}\label{alg:genTSD}
\begin{algorithmic}[1]

\Function{ \genTISD{} }{$q_{p}$, $\overline{q}_{p}$, environment model \eqref{eqn: quantity_of_interest}}
    \For{$p \in \{1, ..., N_p\}$}
        \State $k \gets \sqrt{ \frac{ \sum_{i \in \Rcal} \frac{C(x^i)^2}{R(x^i)} }{Q} }$
        \State $q_{\infty} \gets k/(k+1)$
        \State $\overline{q} \gets \min (\overline{q}_{p} , q_{\infty, p} - \epsilon)$
        \State $\phi_p \gets \Delta T(\overline{q}, q_p)$ using \eqref{eqn:multi_clarity_inverse}
    \EndFor
    \State $\phi_{p} \gets \phi_{p}/({\sum_{p = 1}^{N_p} \phi_{p}}), \quad \forall p \in \{1, ..., N_p\}$
    \State \Return $\phi_{p} \  \forall p \in \{1, ..., N_p \}$
\EndFunction
\end{algorithmic}
\end{algorithm}
Although the proposed method uses centralized decision-making, we distribute computation across the network to enable real-time operation. \cref{fig:mEclares_overview} provides a high-level view of the architecture, and \cref{fig:comms} illustrates two supported communication models. Construction of the TISD, multi-agent ergodic trajectory generation, and the scheduling component of \Rmesch{} are executed on a central base computer.

Each rechargeable robot generates a single \emph{candidate trajectory} onboard, which attempts to track a portion of the ergodic trajectory before reaching the charging station. All candidate trajectories are sent to the base computer, where the \Rmesch{} algorithm jointly evaluates them and determines which trajectories to commit, based on energy feasibility and coordination requirements. This setup enables decentralized trajectory generation at the robot level while maintaining global coordination through centralized scheduling.

\subsection{Method Organization}

% In this work, we adopt PTO directly from \cite{Dressel_Ergodic}. 

In the next sections, we describe the \mEclares{} framework in detail. We begin with \genTISD{}, a method for generating the target information spatial distribution (TISD) used in multi-agent ergodic search. We then present the details of the \Rmesch{} algorithm. We also establish notation for trajectories. Let $x^i([t_k, t_k + T_H]; t_k, x^i_k)$ represent the ergodic trajectory for the $i^{\text{th}}$ rechargeable robot at time $t_k$, starting from state $x^i_k$ and defined over a time horizon of $T_H$ seconds. We denote this as $x^{i, \text{ergo}}_k$. The same notation applies to other trajectories. An overview of the notation is provided in \cref{table:1}. Without loss of generality, we present our method assuming $N$ rechargeable robots modeled as quadrotors and one mobile charging rover.

\section{Generate Target Spatial Distribution (\genTISD{})}

The \genTISD{} algorithm is described in \cref{alg:genTSD}. Let $\phi_p$ denote the target information density evaluated for cell $p$. At the $k$-th iteration (i.e, at time $t_k = k T_H$), we set $\phi_p$ to be the time that the robot would need to increase the clarity from $q_p(t_k)$ to the target $\overline q_p$ by observing cell $p$ (Lines 3-6). This is determined using~\eqref{eqn:multi_clarity_inverse}. The small positive constant $\epsilon > 0$ in Line 5 ensures that target clarity is always less 
than the maximum attainable clarity, i.e., $\overline q_p < q_{\infty, p}$. Finally, we normalize $\phi_p$ such that the sum of $\sum_{p \in \Pcal_{\text{cells}}} \phi_p = 1$ (Line 8). Once $\phi$ is constructed, trajectory optimization solvers can be used to generate the ergodic trajectories $x^{i, ergo}_k, \forall i \in \Rcal$.

\section{Mobile Charging Station Nominal Trajectory}

To support coordination with the team of rechargeable robots, we generate a nominal trajectory for the mobile charging station that tracks the geometric center of the team’s nominal ergodic trajectories. At each decision point $t_k$, the geometric center of the team’s ergodic trajectories is defined as:
\eqn{
x^{\text{cent}}(t) = \frac{1}{N} \sum_{i=1}^N x^{i, \text{ergo}}_k(t), \quad \forall t \in [t_k, t_k + T_H].
}
At time $t_k$, the mobile charging nominal trajectory $x_{k}^{c, nom}$, defined over the time interval [$t_{k}, t_{k, H}$], is generated by solving the following optimal control problem:  
\begin{subequations}
\label{eq:nom_charging_Station}
\eqn{
        \min_{x^c(t),u^c(t)} \ &  \int_{t_{k}}^{t_{k,H}} \norm{x^c(t) - x^{\text{cent}}(t)}_{\mathbf{Q}}^2 + \norm{u^c(t)}_{\mathbf{R}}^2 dt \label{eq:rover_centroid}\\
        \textrm{s.t. } \ & x^c(t_{k}) = x_{k}^{c, nom}(t_{k}) \\
        &  \Dot{x}^c = f_r^c(x^c, u^c)
}
\end{subequations}
where $\mathbf{Q} \in \mathbb{S}^{n}_{++}$ and $\mathbf{R} \in \mathbb{S}^{m}_{++}$ weights state cost and  control cost respectively.
 This formulation ensures that the mobile charging robot stays centrally positioned relative to the rechargeable robots without requiring explicit communication or coordination, enabling robust support for persistent operation. 

\section{Robust Multi-Agent Energy-Aware Scheduling for Task Persistence (\Rmesch{})}

To facilitate readers, we organize the presentation of \Rmesch{} into three subsections: \cref{sec:RmeSch_mot} introduces the motivation behind \Rmesch{} and outlines its key ideas,  \cref{sec:RmeSch_nom} describes the method in detail, and \cref{sec:RmeSch_guarantees} discusses theoretical guarantees around \Rmesch{}.

\subsection{\Rmesch{} Motivation and Key Ideas}\label{sec:RmeSch_mot}

As a low-level module, \Rmesch{} ensures task persistence. The solution follows three steps, with the \Rmesch{} module running every $T_E$ seconds at discrete time steps $t_j = jT_E$, where $j \in \mathbb{Z}_0$:

\begin{itemize}
    \item Compute the rendezvous point where the rechargeable robot will return for recharging.
    \item Determine the reserve energy required at the rendezvous to account for uncertainty in the charging station's position.
    \item Construct a trajectory that follows a portion of the ergodic trajectory before reaching the rendezvous point. We refer to this as the \textit{candidate trajectory}.
    \item Commit the candidate trajectory if it satisfies both the minimum energy condition~\eqref{eq:min_energy_constraints} and the minimum gap condition~\eqref{eq:gap_lower_bound}. The result is the \textit{committed trajectory}.
    \item Along with the committed trajectory, a fail-safe return schedule is generated based on the current SoC level to ensure safe return in case of central node failure.
\end{itemize}
Before detailing each step, we first explain how \Rmesch{} evaluates the satisfaction of conditions \eqref{eq:min_energy_constraints} and \eqref{eq:gap_lower_bound}. This is one of our key contributions, and we explain it by first discussing its motivation and then describing its mechanism.

\begin{figure} [t]
  \centering
  \includegraphics[width=1.0\columnwidth]{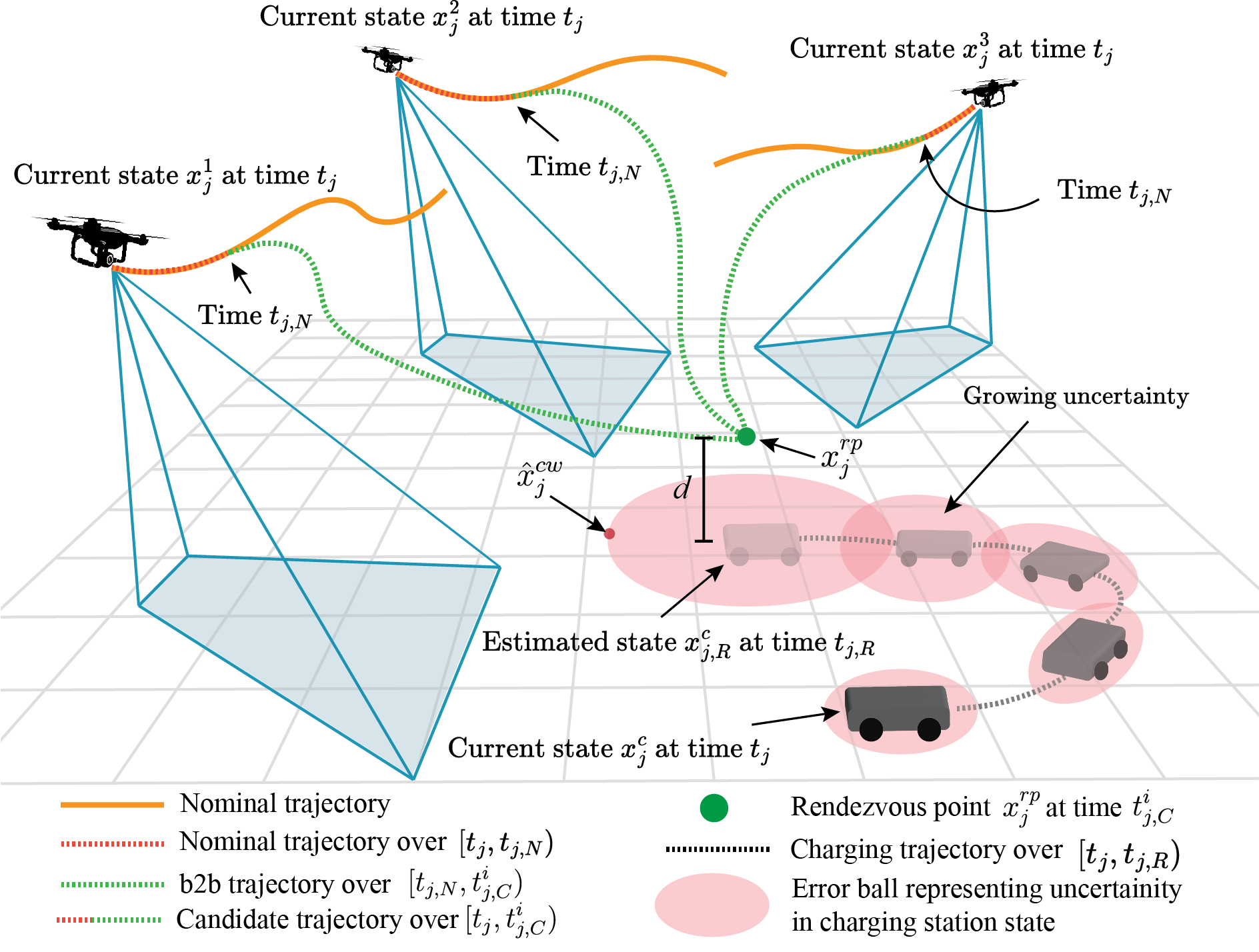}
  \caption{This figure illustrates the generation of candidate trajectories at time $t_j$. All the candidate trajectories terminate at the rendezvous point $x^{rp}_j$ at time $t^i_{j,C}$.}
  \vspace{-20pt}
  \label{fig:traj_eware}
\end{figure}
\subsubsection{\Rmesch{} Motivation}

Consider $N$ quadrotors sharing a mobile charging rover, as shown in \cref{fig:traj_eware}. To prevent charging conflicts, we propose a scheduling method based on two principles.

First, if multiple robots are predicted to arrive simultaneously, one is rescheduled to arrive earlier using gap flags explained below. Second, if robots visit the charging station at different times due to varying discharge profiles, the algorithm checks that each robot has enough energy to continue its mission, ensuring that the minimum energy condition is never violated.

To implement this approach, we introduce two modules: \gware{} and \eware{}. The \gware{} module enforces the minimum time gap between consecutive charging sessions by constructing \textbf{\textit{gap flags}} and resolving conflicts by selecting the robot with the least remaining flight time to return first—similar to dropping a constraint to restore feasibility. This allows the remaining robots to maintain the desired gap defined by \eqref{eq:gap_lower_bound}. Once the gap flags are satisfied, the \eware{} module checks whether each robot has enough energy to continue its mission, ensuring that the minimum energy condition \eqref{eq:min_energy_constraints} is also satisfied. 
\begin{table}[h!]
\scriptsize
% scriptsize
\setlength\extrarowheight{-3pt}
\begin{tabular}{p{0.10\linewidth} |
p{0.75\linewidth}}
\midrule
\textbf{Symbol} & \quad \quad \quad \quad \quad \quad \quad \textbf{Definition}  \\ 
\midrule
\multicolumn{2}{l}{Indices} \\
\midrule
$i$ & Rechargeable robot index \\
$j$ & \Rmesch{} iteration index  \\
$k$ & Nominal trajectory planner iteration index  \\
$l$ & Rechargeable robot index in sorted list \\
\midrule
\multicolumn{2}{l}{Constant shared time horizons} \\
\midrule
% $T_{\delta}$ & Recharging duration  \\
$T_{\delta}$ & $T_{\delta} = T_{ch} + T_{bf}$ Charging + Buffer time  \\
$T_N$ & Nominal trajectory horizon of the rechargeable robot available at time $t_j$  \\
$T_R$ & Charging robot nominal trajectory horizon available at $t_j$ / Time taken by the rechargeable robot to reach the charging station \\
$T_E$ & Time interval between $j$ and $j + 1$ iteration  \\

% $\Delta T_{R,j}$ &  $T^i_{R,j-1}$ - $T_{R,j}$ \\

\midrule
\multicolumn{2}{l}{Dynamic time horizons for robot $i$ computed at $t_j$} \\
\midrule
$T^i_{L,j}$ & Worst-case landing time  \\
$T^i_{C,j}$ & Candidate trajectory ($T^i_{C,j} = T_R - T^i_{L,j}$) \\
$T^i_{B,j}$ & back-to-base trajectory ($T^i_{B,j} = T^i_{C,j} - T_N$)  \\
$T^l_{F,j}$ &  Remaining battery time of the $l^{th}$ robot in the sorted list at time $t_j$ \\
\midrule
\multicolumn{2}{l}{Time points} \\
\midrule
$t_j$ & Start time of iteration $j$ \\
$t_{j,N}$ & $t_j + T_N$ \\
$t^i_{j,C}$ & $t_j + T^i_{C,j}$ \\
$t_{j,R}$ & $t_j + T_R$ \\
$t_{m}^i$ &  $m^{th}$ time $i^{th}$ robot returns for recharging   \\
\midrule
% \multicolumn{2}{l}{Trajectories} \\
% \midrule
% $x_{j}^{i, nom}$ & $i^{th}$ robot nominal trajectory at $j$ \\
% $x_{j}^{i, b2b}$ & $i^{th}$ robot b2b trajectory at $j$ \\
% $x_{j}^{i, can}$ & $i^{th}$ robot candidate trajectory at $j$ \\
% $x_{j}^{i, com}$ & $i^{th}$ robot committed trajectory at $j$ \\
% \midrule
\end{tabular}

\caption{Time and Index Notation at a glance}
\vspace{-23pt}
\label{table:1}
\end{table}
\subsubsection{Key Idea: Construction of Gap flags}

We begin by describing the construction of \textit{gap flags} and their role in preventing charging conflicts. At each iteration of \Rmesch{}, rechargeable robots are sorted by their remaining flight time into the ordered set $\Rcal' = \{1', \dots, N'\}$, where $1'$ has the least flight time. For each robot $l \in \Rcal' \setminus \{1'\}$, a gap flag is constructed relative to $1'$ as:
\eqn{
G^{l} = T_{F,j}^{l} > (T_R + T_E + lT_{\delta}),
}
where $T_{F,j}^{l}$ is $l^{th}$ robot remaining flight time at time $t_j$, $T_R$ is the time to reach the charging station, $T_E$ is the decision interval, and $T_{\delta}$ includes the charging duration and the buffer time required to resume the mission.

These flags enforce a minimum gap of $lT_{\delta}$ between robot $1'$ and robot $l$ in $\Rcal'$. For example, the minimum gap between the first and third robots is $2T_{\delta}$. If any gap flag is not satisfied, the robot with the least remaining flight time, i.e., $1'$, is rescheduled for recharging. The satisfaction of the gap flag condition guarantees that there will be at least $T_{\delta}$ between successive charging sessions.

\subsection{\Rmesch{} Methodology}\label{sec:RmeSch_nom}

In this section, we present \Rmesch{} in detail. After establishing the construction of gap flags, we demonstrate how they are iteratively checked within the full solution scheme to ensure conditions \eqref{eq:min_energy_constraints} and \eqref{eq:gap_lower_bound} hold for all $t \in [0, \infty)$. We also discuss how the proposed method accounts for the uncertainty in the position of the mobile charging robot. This solution is developed under a few key assumptions: 

\begin{assumption}
\label{assumption:1}
At each iteration of \Rmesch{}, the nominal trajectories of the rechargeable robots are known for $T_N$ seconds, and the nominal trajectory of the mobile charging robot's for $T_R$ seconds. This can be ensured by selecting time horizons such as they satisfy $T_N, T_R < T_H$. 
% \footnote{The method to determine $T_N$ and $T_C$ is described later.}
\end{assumption}

% \end{center}
\subsubsection{Estimating Rendezvous Point}
\label{renp}
At the $j^{th}$ iteration of \Rmesch{}, we estimate the mobile charging robot's position at $t_j + T_R$, i.e., $\hat{x}^c(t_{j, R})$, and place the rendezvous point $d$ meters above it. The rechargeable robots will return to this point, as shown in \cref{fig:traj_eware}.

Given the current state estimate $\hat{x}^c(t_j)$ and its covariance $\Sigma^c(t_j)$ from the EKF, we use the EKF predict equations \cite{applied_optimal_estimation} to compute the mobile charging robot state estimate at $t_{j,R}$, i.e. $\hat{x}^c(t_{j,R})$ and $\Sigma^c(t_{j,R})$. The rendezvous point $x^{rp}_j \in \R^n$ is then computed as follows:
\eqn{
x^{rp}_j = 
    \begin{bmatrix}
        \Psi(\hat{x}^c(t_{j,R})) \\
        \textbf{0}_{n-2} 
    \end{bmatrix}
    +
    \begin{bmatrix}
        \textbf{0}_{2} \\
        d \\
        \textbf{0}_{n-3}
    \end{bmatrix}    
% [\Psi(\hat{x}^c(t_{j,R})); 0_{(n-2)}] + [0_{2}; d; 0_{(n-3)}]
}
where  $\Psi : \mathbb{R}^{c} \to \mathbb{R}^{2}$ is a mapping that
returns the 2-D position coordinates, $d \in \mathbb{R}_{> 0}$ is added to the z-dim of the state, and $n$ is the rechargeable robot state dimension \eqref{eqn:rechargeable}. The rendezvous point corresponds to the hover reference state $x_{j}^{rp}$ for the rechargeable robot, positioned $d$ meters above the predicted position of the mobile charging robot.

%% =====================================================================

\subsubsection{Reserve Energy for Uncertainty-Aware Landing}

Along with the rendezvous point $x^{rp}_j$, we also compute the remaining energy the robot must have at the rendezvous point to account for uncertainty in the mobile charging robot's position for landing. This corresponds to the energy cost of going from rendezvous point $x^{rp}_j$ to the furthest state $\hat{x}^{cw}_j$ within the 95\% confidence interval covariance ellipse.
    % \State $T_{\text{ret}}^{k} \gets T_{\text{fail}} + (k-1)(T_\delta), \quad \forall k \in \Rcal'$

Now, we compute the furthest point on the boundary of the 95\% confidence ellipse as follows:
\eqn{
\hat{x}^{cw}_j = \hat{x}^c(t_{j,R}) + q_{max} \sqrt{\chi^2_{c, 0.95} \lambda_{max}} 
}
where $\lambda_{max} \in \R$ is the largest eigenvalue of the covariance matrix $\Sigma^c(t_{j,R})$, $q_{max} \in \R^c$ is the eigenvector corresponding to $\lambda_{max}$, and $\chi^2_{c, 0.95}$ corresponds to the value from the chi-squared distribution with $c$ degrees of freedom in the 95\% confidence interval. To compute the reserve energy, we formulate the following problem $\forall i \in \Rcal$:
\begin{subequations}
\label{eq:landing_overall}
\eqn{
        \min_{\chi^i(t),u^i(t), t^i_{f}} \ &  t^i_{f} \label{eq:landing_prob}\\
        \textrm{s.t. } \ & \chi^i(t^i_{0}) = \chi^i_{rp} \\
        &  \Dot{\chi} = f_r^i(\chi^i, u^i) \\
        & x^i(t^i_{f}) = \hat{x}^{cw}_j
}
\end{subequations}
where $\chi^i_{rp} = [[x^{rp}_j]^T, e^i_0]^T $ is the initial system state comprising of $x^{rp}_j \in \R^n$ and the energy $e_0 \in \mathbb{R}_{> 0}$. The reserve energy $e^{i, res}_{j}$ and landing time $T^i_{L,j}$ are computed as follows:
\begin{subequations}
    \eqn{e_{j}^{i, res} &= e^i(t^i_{f}) - e^i(t^i_{0}) \quad \forall i \in \Rcal \\
    T^i_{L,j} &= t^i_{f} - t^i_{0} \quad \forall i \in \Rcal
    }
\end{subequations}

% ========================================================================

\subsubsection{Construction of Candidate Trajectories} 

Now, we generate the candidate trajectories for all rechargeable robots to reach the rendezvous point $x^{rp}_j$ from the current state $x^i(t_j)$ within $T^i_{C,j} = T_R - T^i_{L,j}$ s.
\begin{algorithm}
\footnotesize
\caption{The \Rmesch{} algorithm}\label{alg:mesch}
\begin{algorithmic}[1]

\Function{\Rmesch{}}{$x_j^{i,\text{can}},\ x_{j-1}^{i,\text{com}},\ e_j^{i,\text{res}}$}
    \If{$x_j^{i,\text{can}},\ x_{j-1}^{i,\text{com}},\ e_j^{i,\text{res}}$ not received for all $ i \in \Rcal$ }
        \State \Return \Rmesch($x_j^{i,\text{can}},\ x_{j-1}^{i,\text{com}},\ e_j^{i,\text{res}}$)
    \EndIf
    \State \texttt{GapVio}, $x_{j}^{i,\text{com}}, \Rcal'$ $\gets$ \gware{($x_j^{i,\text{can}},\ x_{j-1}^{i,\text{com}}$)}
    \State $ret^i_j \gets \text{index}(i, \Rcal') \quad \text{s.t. } \Rcal'[l] = i$ \Comment{Index of robot $i$ in $\Rcal'$}
    \If{\texttt{GapVio} $==$ 1}
        \State \textbf{Publish} $\text{mobcon}_j$ = True \Comment{publishes message to mobile charging station to continue the mission}
        \State \Return $x_{j}^{i,\text{com}},  l^i \quad \forall i \in \Rcal$
    \EndIf
    \State $x_{j}^{i,\text{com}} \gets$ \eware{($x_j^{i,\text{can}},\ x_{j-1}^{i,\text{com}},\ e_j^{i,\text{res}}$)}
    \State \textbf{Publish}  $\text{mobcon}_j$ = True 
    \State \Return $x_{j}^{i,\text{com}},  ret^i_j  \quad \forall i \in \Rcal$
\EndFunction
\end{algorithmic}
\end{algorithm}
Given nominal trajectories $x^{i, nom}_j$ $\forall i \in \Rcal$, we construct a candidate trajectory that tracks a portion of the nominal trajectory for $T_N$ s and then reaches the rendezvous point $x^{rp}_j$ within $T^i_{B,j}$ = $T^i_{C,j} - T_N$ s. For the $i^{th}$ rechargeable robot, the candidate trajectory is constructed by concatenating the nominal trajectory with a \textbf{\textit{back-to-base (b2b) trajectory}}. 
Let the $i^{th}$ rechargeable robot state at time $t_j$ be $x^i_j \in \Xcal$ and the system state at $t_j$ be $\Xcal^i_j \in \Zcal^i_r$. We construct a b2b trajectory $x^{i,b2b}_j$ defined over interval [$t_{j,N}, t^i_{j,C}$] by solving:
\begin{subequations}
\label{eq:b2b_prob_overall}
\eqn{
        \min_{x^i(t),u^i(t)} \ &  \int_{t_{j,N}}^{t^i_{j,C}} \norm{x^i(t) - x^{rp}_j}_{\mathbf{Q}}^2 + \norm{u^i(t)}_{\mathbf{R}}^2 dt \label{eq:b2b_prob}\\
        \textrm{s.t. } \ & x^i(t_{j,N}) = x_{j}^{i, nom}(t_{j,N}) \\
        &  \Dot{x}^i = f_r^i(x^i, u^i) \\
        & x^i(t^i_{j,C}) = x^{rp}_j
}
\end{subequations}

where $\mathbf{Q} \in \mathbb{S}^{n}_{++}$ and $\mathbf{R} \in \mathbb{S}^{m}_{++}$ weights state cost and  control cost respectively.

Once b2b trajectory $x_{j}^{i,b2b}$ is generated, we numerically construct the system candidate trajectory
\eqn{
    \chi^{i,can}_j &= \begin{cases} x^{i,can}_j(t), & t \in [t_j, t^i_{j,C}) \\
        e^{i,can}_j(t), & t \in [t_j, t^i_{j,C})
   \end{cases} 
}over a time interval [$t_j, t^i_{j,C}$) by solving the initial value problem for each rechargeable robot system, i.e.
\begin{algorithm}
\footnotesize
\caption{The \gware{} algorithm}\label{alg:gware}
\begin{algorithmic}[1]
\Function{\gware{}}{$x_j^{i,\text{can}},\ x_{j-1}^{i,\text{com}}$}
    \State Sort $x_j^{i,\text{can}}$ based on $T_F^i$
    \For{$l \in \Rcal' \setminus \{1'\} = \{2', \dots, N'-1\}$}
        \State $G^l \gets (T_F^l - T_R - T_E) > l(T_\delta)$
        \If{$G^l == 0$ \textbf{and} $l.\text{charging} \neq \text{1}$}
            \State $x_j^{1',\text{com}} \gets x_{j-1}^{1',\text{com}}$
            \State $x_j^{l,\text{com}} \gets x_j^{l,\text{can}}$ for all $l \in \Rcal'$
            % \State Initiate landing for the $1'$-th robot at $t^i_{j,C}$
            \State \Return \texttt{True}, $x_j^{i,\text{com}}$, $\Rcal'$ 
        \EndIf
    \EndFor
    \State \Return \texttt{False}, $x_j^{i,\text{com}}$, $\Rcal'$
\EndFunction
\end{algorithmic}
\end{algorithm}
\begin{algorithm}
\footnotesize
\caption{The \eware{} algorithm}\label{alg:eware}
\begin{algorithmic}[1]

\Function{\eware{}}{$x_j^{i,\text{can}},\ x_{j-1}^{i,\text{com}},\ e_j^{i,\text{res}}$}
    \For{$i \in \{1, \dots, N\}$}
        \If{$e^i(t) \geq e_j^{i,\text{res}} \quad \forall t \in [t_j, t^i_{j,C}]$}
            \State $x_j^{i,\text{com}} \gets x_j^{i,\text{can}}$
        \Else
            \State $x_j^{i,\text{com}} \gets x_{j-1}^{i,\text{com}}$
        \EndIf
    \EndFor
\EndFunction

\end{algorithmic}
\end{algorithm}
\begin{subequations}
\eqn{
\label{eqn:candidate}
    \Dot{\chi}^i &= f(\chi^i, u^i(t)),\\
    \chi^i(t_j) &= \chi^i_j \\
    u^i(t) &= \begin{cases} \pi^i_r(\chi^i,  x^{i,nom}_j(t)), & t \in [t_j, t_{j,N}) \\
        \pi^i_r (\chi,  x_{j}^{i,b2b}(t)), &  t \in [t_{j,N}, t^i_{j,C})
   \end{cases} 
}
\end{subequations}

where $\pi^i_r: \mathcal{Z}_r^i \times \mathcal{X}_r^i \rightarrow \mathcal{U}_r^i$ is a control policy to track the portion of the nominal trajectory and the b2b trajectory. \Cref{fig:traj_eware} shows the candidate trajectory generation process with 3 rechargeable robots and 1 mobile charging robot.

% \vspace{-12pt}
\subsubsection{Robust Energy-aware Scheduling}

Given the candidate trajectory and the reserve energy for each rechargeable robot $i \in \Rcal$, we check if the minimum SoC condition \eqref{eq:min_energy_constraints} and the minimum gap condition \eqref{eq:gap_lower_bound} are satisfied throughout the candidate trajectory. The overall algorithm described in \cref{alg:mesch} consists of the two subroutines: \gware{} (gap-aware) and \eware{} (energy-aware).

\subsubsection{Gap-aware (\gware)}
\gware{} described in \cref{alg:gware} checks if the rechargeable robots would continue to have the gap of $T_{\delta}$ seconds between their expected returns if the candidate trajectories were committed.  

(Lines 2-4) Here the gap flags are constructed for each $l^{th}$ robot that is not currently charging or returning, relative to the first robot in the sorted list (the $1^{\text{st}}$ robot)
\eqn{
\label{gap_flag}
G^{l} = T_{F,j}^{l} > (T_R + T_{E} + lT_{\delta})
}
Satisfaction of the gap flag condition at the $j^{th}$ iteration implies that rechargeable robots are estimated to have at least $T_{\delta}$ of the gap between their expected returns over the time interval $[t_j, t_{j,R})$, i.e. $\forall t^{i_1}_{m_1},t^{i_2}_{m_2}  \in \Tcal $:  
    \eqn{
&T_{F,j}^{l} > (T_R + T_{E} + lT_{\delta}) \\
&\implies |t^{i_1}_{m_1} - t^{i_2}_{m_2}| > T_{\delta}   &\forall t \in [t_{j}, t_{j, R} )
}
(Lines 5-7) If any gap flags are false, the committed trajectory of the first rechargeable robot in the sorted list remains unchanged, and it returns to the charging station. Meanwhile, the candidate trajectories are committed for the subsequent rechargeable robots in the sorted list. 

% By bringing bacl the first robot in the sorted list early, we ensure that the second robot in the sorted list has at least $T_{\delta}  + T_\delta$ seconds of battery time before it returns to the charging robot. 

\subsubsection{Energy-aware (\eware)} If no gap flag violations occur, indicating that all rechargeable robots have sufficient gaps between their expected return for recharging, we proceed to check if each robot has adequate energy to continue the mission using \eware{} described in \cref{alg:eware}. 

(Lines 3-6) We assess whether each rechargeable robot can reach the charging station without depleting its energy below the reserve level while following the candidate trajectory. We refer to this condition as the \textbf{\textit{Reserve SoC Condition}}:
\eqn{
\label{min_soc}
e^i(t) > e^{i, res}_j \ \forall t \in [t_j, t^i_{j,C}]
}
If successful, the candidate trajectory replaces the current committed one. For the returning robot, a landing controller is assumed to exist:
\begin{assumption} When the returning rechargeable robot reaches rendezvous point $x^c_{rp}$ at $t^i_{j,C}$, there exists a landing controller $\pi^i_l : [t^i_{j,C}, t_{j,R}) \times \Xcal_r^i \rightarrow \Ucal$ that guides the rechargeable robot to the mobile charging robot.
% (\cref{alg:gware}: Line 8) and (\cref{alg:eware}: Line 7)
\end{assumption}

\begin{algorithm}
\footnotesize
\caption{Fail-safe onboard rechargeable robot $i$}\label{alg:onboard_quad}
\begin{algorithmic}[1]
\State \textbf{Try} $(x_j^{i,\text{com}},\ \text{ret}^i_j) \gets$ \Rmesch{($x_j^{i,\text{can}},\ x_{j-1}^{i,\text{com}},\ e_j^{i,\text{res}}$)} until $t_{j-1,N}$ 
\If{successful}
    \State Execute $x_j^{i,\text{com}}$
\Else
    \If{$\text{ret}^i_{j-1} == 1$}
        \State Execute $x_{j-1}^{i,\text{com}}(t)$ over $[t_j,\ t_j + T_R]$
    \Else
        \State Idle (hover) for time $\text{ret}^i_{j-1} (T_\delta)$
        \State Execute $x_{j-1}^{i,\text{com}}(t - \text{ret}^i_{j-1} (T_\delta))$ over time horizon $[t_j + \text{ret}^i_{j-1} (T_\delta),\ t_j + \text{ret}^i_{j-1} (T_\delta) + T_R]$
    \EndIf
\EndIf
\end{algorithmic}
\end{algorithm}
\begin{algorithm}
\footnotesize
\caption{Fail-safe onboard mobile charging robot}\label{alg:onboard_mobile}
\begin{algorithmic}[1]
 \State At time $t_j$  \textbf{Initialize} $\text{mobcontinue}_j$ = False
 \State \textbf{Try} Message $\text{mobcontinue}_j$ received until $t_{j-1,N}$
\If{successful}
    \If{$\text{mobcontinue}_j$ == True}
        \State \textbf{Continue} executing the nominal trajectory
    \EndIf
\Else
    \State Execute stopping at time $t_j + T_R$
\EndIf
\end{algorithmic}
\end{algorithm}
\subsubsection{Fail-safe maneuver planning} 
At each iteration of \Rmesch{}, we also plan a fail-safe maneuver to handle scenarios involving central node failure or communication delays. Specifically, \Rmesch{} transmits to each rechargeable robot a message indicating whether to commit the new trajectory. Along with this, each robot is assigned a return position in the sorted list $\Rcal'$, which is generated based on remaining flight time. This position provides each robot with its rank in the return sequence in case no message is received due to failure.
The logic executed onboard each rechargeable robot $i$ is detailed in \cref{alg:onboard_quad} and the fail-safe logic for the mobile charging robot is detailed in \cref{alg:onboard_mobile}. \Cref{fig:central_node_fail} illustrates the behavior of the fail-safe maneuver.

\textbf{\textit{1) Fail-safe maneuver onboard rechargeable robot}}
(Lines 1--3) Once the candidate trajectories are generated onboard, each rechargeable robot transmits its candidate trajectory to the base computer, which executes \Rmesch{} to determine whether the trajectory should be committed. Starting at time $t_j$, each robot awaits a response from the base until $t_{j-1,N}$. Nominally, the robot should receive this message by $t_{j,T+E} < t_{j-1,N}$; however, due to potential communication delays, a response may arrive later. If a valid response is received by $t_{j,N}$, the robot executes the new committed trajectory.

 \begin{figure*}[t]
  \centering
  \includegraphics[width=2.1\columnwidth]{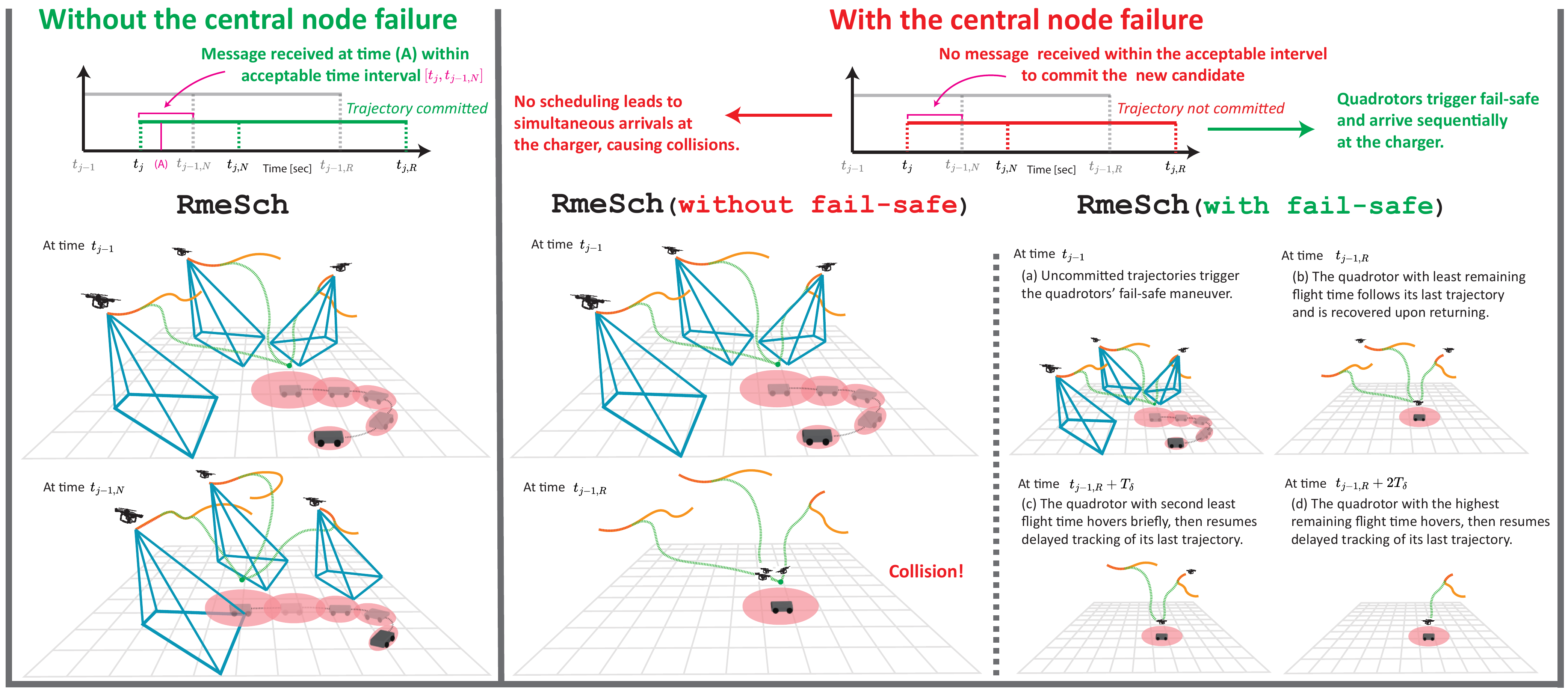}              
  \caption{The figure shows the behavior of the algorithm when the central node fails.}
\label{fig:central_node_fail}
\end{figure*}
(Lines 4--9) If no message is received from the central node by time $t_{j,N}$, robot $i$ executes a fallback maneuver using its previously committed trajectory $x_{j-1}^{i,\text{com}}$ and its return index $\text{ret}_{j-1}^i$ from the previous decision epoch. If $\text{ret}_{j-1}^i = 1$, the robot continues executing its last committed trajectory:
\eqn{
x_j^{i,\text{com}}(t) = x_{j-1}^{i,\text{com}}(t), \quad t \in [t_j,\ t_j + T_R]
}
If $\text{ret}_{j-1}^i > 1$, the robot remains idle (e.g., hovers) for $\text{ret}_{j-1}^i (T_\delta)$ seconds, and then begins executing a time-shifted version of its previous trajectory:
\eqn{
\begin{aligned}
    x_j^{i,\text{com}}(t) = x_{j-1}^{i,\text{com}}(t - \text{ret}_{j-1}^i T_\delta), \\
    t \in [t_j + \text{ret}_{j-1}^i T_\delta,\ t_j + \text{ret}_{j-1}^i T_\delta + T_R]
\end{aligned}
}
This fallback guarantees mutually exclusive access to the charging station and ensures energy-feasible operation even in the absence of centralized coordination.

\textbf{\textit{2) Fail-safe maneuver onboard mobile charging robot}}

To detect central node failures, the mobile charging robot monitors the $\text{mobcontinue}_j$ message. If this message is received by $t_{j-1, T_N}$, the robot continues executing its nominal trajectory. Otherwise, if the message is not received by the deadline, it halts the mission at time $t_{j,R}$.

% In the event that a robot does not receive an updated committed trajectory from the central node by time $t_j + T_{\text{fail}}$, it executes a fallback maneuver using its previous committed trajectory. Let $\text{ret}^i = k$ denote robot $i$'s position in the return order $\Rcal'$, where lower values correspond to higher priority. The fallback policy is defined as:

% \[
% x_j^{i,\text{com}}(t) = 
% \begin{cases}
% x_{j-1}^{i,\text{com}}(t), & \text{if } k = 1,\ t \in [t_j,\ t_j + T_R] \\
% x_{j-1}^{i,\text{com}}(t - k T_\delta), & \text{if } k > 1,\ t \in [t_j + k T_\delta,\ t_j + k T_\delta + T_R] \\
% \text{idle}, & \text{if } t < t_j + k T_\delta
% \end{cases}
% \]

% Robot $i$ remains idle (e.g., hovering) until its assigned return time $t_j + k T_\delta$ and then begins following a time-shifted version of its last committed trajectory. This ensures mutual exclusion at the charging station while maintaining minimum energy feasibility in the absence of central coordination.

\subsection{\Rmesch{} Theoretical Guarantees}\label{sec:RmeSch_guarantees}
This section provides the theoretical conditions under which \Rmesch{} guarantees feasibility, robustness, and adaptability. We begin by deriving an upper bound on the number of robots that can be supported at mission start based on the minimum remaining flight time. We then present a general feasibility theorem that guarantees all robots return safely with the required energy and time separation under both nominal conditions and central node failure. Finally, we discuss robustness by addressing two key scenarios: rechargeable robot failure and the addition of new robots during the mission, and provide conditions under which feasibility is preserved in both cases.

\subsubsection{Upper bound on number of robots}

\begin{lemma}
\label{lemma1}
At iteration $j = 0$, given the sorted list of remaining flight times $\{T_{F,0}^{1'}, T_{F,0}^{2'}, \dots\}$ where $T_{F,0}^{1'}$ is the minimum remaining flight time, the maximum number of robots that can be safely supported by the mission while satisfying all gap flags is
\eqn{
N^* = 1 + \left\lfloor \frac{T_{F,0}^{1'} - T_R - T_E}{T_\delta} \right\rfloor,
}
where $T_R$ is the time a rechargeable robot takes to reach the charging station, $T_E$ is the iteration interval, and $T_\delta$ is the minimum required gap between two consecutive returns.
\end{lemma}

\begin{proof} Please see Appendix A Proofs \cref{appen:lemma1_proof} for detailed proof. 
\end{proof}

\subsubsection{Feasibility guarantees}
 \begin{figure*}[t]
  \centering
  \includegraphics[width=2.1\columnwidth]{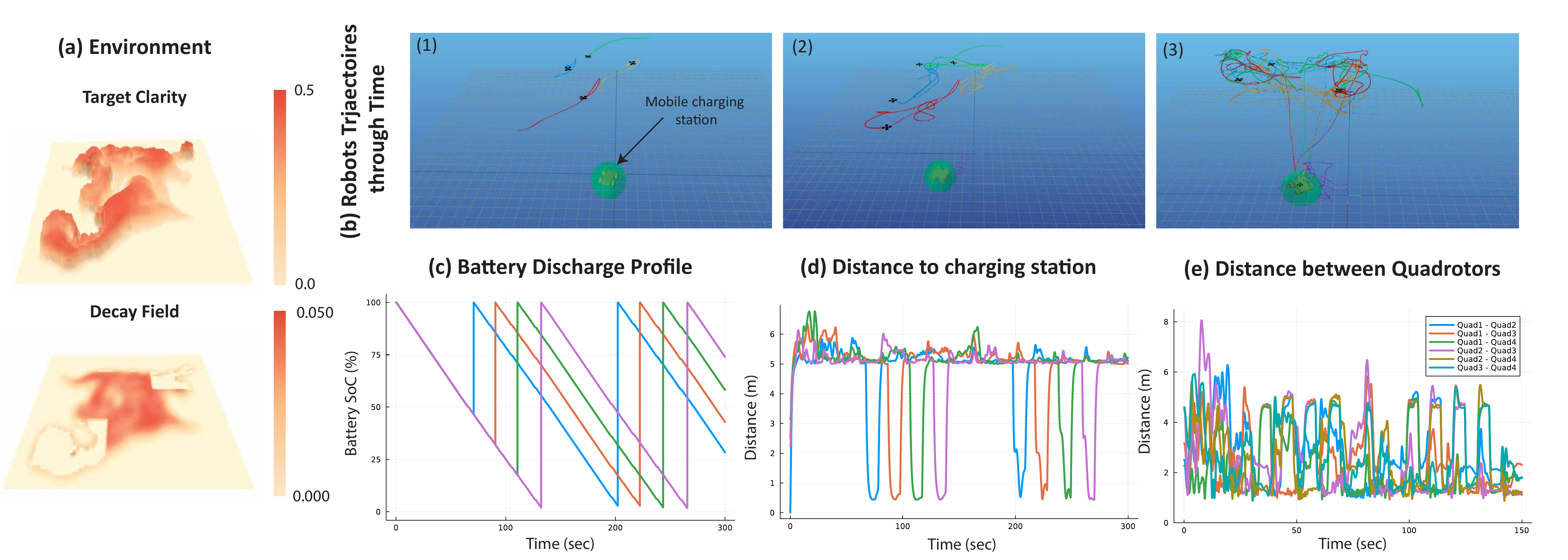}
  \caption{Demonstration of \mEclares{} through a case study: the rechargeable quadrotors track ergodic trajectories that are replanned every 30 seconds, while the mobile charging rover follows the geometric center of the nominal ergodic trajectories of all rechargeable robots.}
  \label{fig:meSch_Ergodic}
\end{figure*}

\begin{thm}
\label{thm1}
Given $|\Rcal|~\leq N^{*}$ derived in \Cref{lemma1}, suppose that at $j = 0$, the Gap flag condition \eqref{gap_flag} and the Reserve SoC condition \eqref{min_soc} are satisfied. Then, the minimum energy constraint \eqref{eq:min_energy_constraints} and the return gap condition \eqref{eq:gap_lower_bound} hold for all $t \in [t_0, t^i_{0, R})$.
For all subsequent iterations $j \geq 1$, if solutions for  \eqref{eq:landing_overall} and \eqref{eq:b2b_prob_overall} exist, and the committed trajectories $x_j^{i, \text{com}} = x^i([t_j, t^i_{j,C}]; t_j, x^i_j)$ are computed for all $i \in \Rcal$ using \cref{alg:mesch}, \cref{alg:onboard_quad} and \cref{alg:onboard_mobile}, then the conditions \eqref{eq:min_energy_constraints} and \eqref{eq:gap_lower_bound} are satisfied for all $t \in [t_j, t_{j-1,R})$ and for all $j \in \posintegers$.
\end{thm} 
\begin{proof} Please see Appendix A Proofs \cref{appen:proof_theorem_1} for detailed proof. 
\end{proof}

\subsubsection{Robustness to rechargeable robot failures and additions}

\begin{remark}
\label{remark}
Given \Cref{lemma1}, \Cref{thm1} also applies in the case when a subset of robots in $\Rcal$ fail during the mission execution. If such failures are detected and the corresponding robots are excluded from future gap flag evaluations, then the minimum energy condition \eqref{eq:min_energy_constraints} and the minimum return gap condition \eqref{eq:gap_lower_bound} continue to hold for the remaining robots.
\end{remark}

% \begin{proof} Please see Appendix A Proofs \cref{appen:cor_1_proof} for detailed proof. 
% \end{proof}

% \begin{prop}
% \label{prop2}
% Let $j$ denote a decision iteration during the mission, and let $T_{F,j}^{1'}$ denote the minimum remaining flight time among all robots in $\Rcal$ at time $t_j$. A new robot can be added to the mission at iteration $j$ if the following condition holds:
% \eqn{
% N_{\text{current}} + 1 \leq 1 + \left\lfloor \frac{T_{F,j}^{1'} - T_R - T_E}{T_\delta} \right\rfloor,
% }
% where $N_{\text{current}}$ is the number of robots currently in the mission, $T_R$ is the return time, $T_E$ is the iteration interval, $T_\delta$ is the minimum required gap between consecutive returns, and $\left\lfloor \cdot \right\rfloor$ denotes the floor operator, which returns the greatest integer less than or equal to its argument.
% \end{prop}

\begin{remark}
\label{remark2}
Following from \cref{lemma1}, at any decision iteration $j$, a new robot can be safely added to the mission if the minimum remaining flight time satisfies
\eqn{
T_{F,j}^{1'} \geq T_R + T_E + (N_{\text{curr}}) \cdot T_\delta,
}
where $T_{F,j}^{1'}$ is the minimum remaining flight time at time $t_j$ and $N_{\text{curr}}$ is the number of robots currently in the mission. This condition ensures that the updated team remains within the admissible bound derived in \cref{lemma1} and all robots satisfy the return gap \eqref{eq:gap_lower_bound} and energy \eqref{eq:min_energy_constraints} constraints.
\end{remark}

% \begin{proof} Please see Appendix A Proofs \cref{appen:cor2_proof} for detailed proof. 
% \end{proof}

\section{Results \& Discussion}
\label{sec:results}

 \begin{figure*}[t]
  \centering
  \includegraphics[width=2.1\columnwidth]{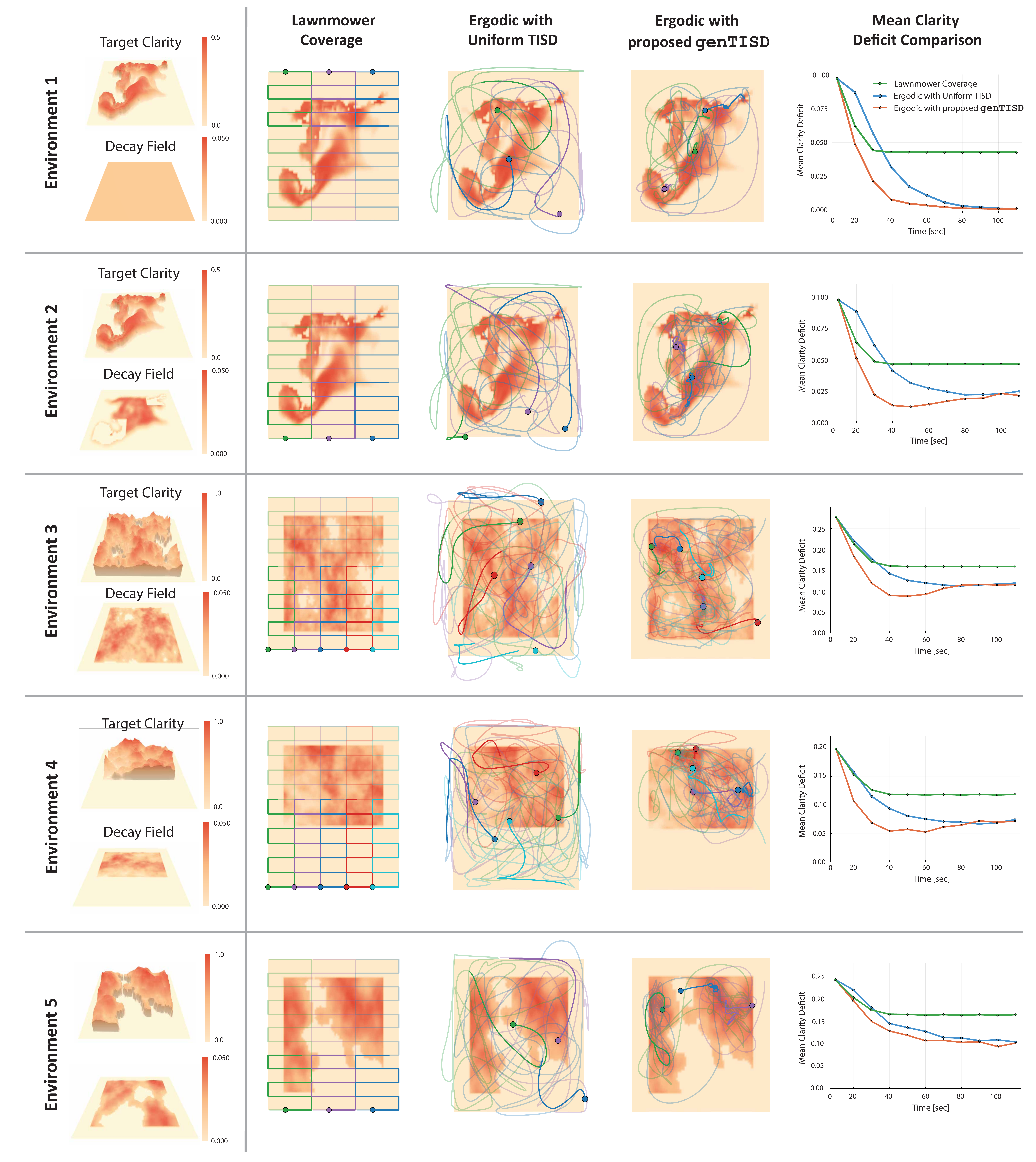}              
  \caption{Comparison of mean clarity deficit over time across five synthetic environments. }
\label{fig:mClares_results}
\end{figure*}
% The proposed \genTISD{} outperforms lawnmower coverage and uniform TISD by dynamically adapting to spatial variations in target clarity and information decay.

In this section, we evaluate \mEclares{} through case studies, baseline comparisons, and hardware experiments. We use quadrotors with 3D nonlinear dynamics from \cite[Eq. (10)]{jackson2021planning} as rechargeable robots and rovers with unicycle models as mobile charging robots. We assume instantaneous recharging ($T_{ch} = 0.0$ s) and a buffer time of $T_{bf} = 15.0$ s, with $T_N = 2.0$ s and $T_R = 18.0$ s, consistent across all experiments.

To generate b2b trajectories, we solve \eqref{eq:b2b_prob_overall} using MPC with the reduced linear quadrotor dynamics from \cite{jackson2021planning}. We use an LQR controller for \eqref{eq:landing_overall} and an LQG controller for landing. Trajectories are generated at 1.0 Hz and tracked at 50.0 Hz with zero-order hold, using the RK4 integration.

\subsection{\textbf{Multi-Agent Energy-Aware Persistent Ergodic Search}}

We evaluate \Rmesch{} by simulating a scenario in which four quadrotors and one rover explore a $10 \times 10$ m domain. The nominal trajectories are collision-free, ergodic paths with a horizon of $T_H = 30.0$ s. All quadrotors follow discharge dynamics given by $\Dot{e} = -0.667$. \Cref{fig:meSch_Ergodic} (a) shows the target clarity distribution and the decay field of the test environment, where the decay field corresponds to the $Q$ values across the domain, as defined in \eqref{multi_clarity_dynamics}. \Cref{fig:meSch_Ergodic} (b) presents still frames from the lightweight UAV simulator, where four quadrotors explore a stochastic spatiotemporal environment. The mobile charging rover tracks the geometric center of the four rechargeable quadrotors.

\Cref{fig:meSch_Ergodic} (c) illustrates the battery discharge profiles of the quadrotors, while \Cref{fig:meSch_Ergodic} (d) shows the distance of each quadrotor to the charging station over time. The results indicate that the quadrotors maintain the minimum required gap between successive visits to the charging station. Collision avoidance is implemented using a potential field method, which generates artificial repulsive forces to steer robots away from each other in real time. To ensure that no more than two quadrotors are on charging-related paths at the same time, the $T_\delta$ parameter is set such that a quadrotor returning from the charging station has sufficient time to rejoin the mission before another begins its return to the charger. Finally, \Cref{fig:meSch_Ergodic} (e) shows the inter-quadrotor distances, confirming that no collisions occur during the mission.
\begin{table*}[h]
    \footnotesize
    \centering
    \caption{Comparison of baseline methods and proposed \Rmesch{}}
    \renewcommand{\arraystretch}{1.3} % Adjust row spacing for better readability
    \begin{tabular}{>{\centering\arraybackslash}p{1.76cm}>{\centering\arraybackslash}p{1.6cm}>{\centering\arraybackslash}p{1.0cm}>{\centering\arraybackslash}p{1.0cm}>{\centering\arraybackslash}p{1.0cm}>{\centering\arraybackslash}p{1.4cm}>{\centering\arraybackslash}p{1.7cm}>{\centering\arraybackslash}p{1.4cm}>{\centering\arraybackslash}p{1.4cm}}
        \hline
        \textbf{Method} & \textbf{Robot Model (Supports Nonlinear Dynamics)} & \textbf{Total Recharging Visits} & \textbf{Gap Violations} & \textbf{Min Energy Violations} & \textbf{Scalability Analysis} & \textbf{Staggered Deployment} & \textbf{Mobile Charging} & \textbf{Central Node failure}\\ 
        \hline
        Baseline 1  & SI (\textcolor{red}{No}) & 8 & 0 & 0 & Not provided & \cellcolor{mygreen}No & \cellcolor{myred}No & \cellcolor{myred}No\\ 
        \hline
        Baseline 2 & Quadrotor (Yes) & 8 & 0 & 0 & Not provided & \cellcolor{myred}Yes & \cellcolor{myred}No & \cellcolor{myred}No\\ 
        \hline
        Baseline 3 & Quadrotor  (Yes) & 8 & \textcolor{red}{2} & 0 & Not provided & \cellcolor{mygreen}No & \cellcolor{myred}No & \cellcolor{myred}No\\ 
        \hline
        Baseline 4 [\mesch{} with only \gware] & Quadrotor (Yes) & 4 & 0 & \textcolor{red}{4} & $\mathcal{O}(N \log N)$ & \cellcolor{mygreen}No & \cellcolor{mygreen}Yes & \cellcolor{myred}No\\ 
        \hline
        Baseline 5 [\mesch{}] & Quadrotor (Yes) & 8 & 0 & 0 & $\mathcal{O}(N \log N)$ & \cellcolor{mygreen}No & \cellcolor{mygreen}Yes & \cellcolor{myred}No\\ 
        \hline
        \textbf{Proposed [\Rmesch{}]} & \textbf{Quadrotor (Yes)} & 8 & 0 & 0 & $\mathcal{O}(N \log N)$ & \cellcolor{mygreen}No & \cellcolor{mygreen}Yes & \cellcolor{mygreen}Yes\\ 
        \hline
    \end{tabular}
    \label{tab:meSch_comparison}
\end{table*}

\subsection{\textbf{Multi-agent Clarity-driven ergodic planner performance comparison to baseline methods}} 
We compare the performance of the proposed method \genTISD{} against two baseline strategies. The first baseline is a lawnmower coverage path~\cite{choset1998coverage}, and the second is the standard uniform TISD used in most ergodic literature~\cite{dong2023time, Mezic_Ergodic}.

Performance is evaluated across five different synthetic environments. Each environment specifies a target clarity distribution and a decay field, where the decay field represents the value of $Q$ across the domain, as defined in \eqref{multi_clarity_dynamics}.

The results, summarized in Figure~\ref{fig:mClares_results}, show that the proposed \genTISD{} consistently achieves a lower mean clarity deficit compared to both baseline methods across all environments. Specifically, the lawnmower coverage strategy exhibits relatively poor performance, especially in environments with nonuniform decay, as it does not adapt to spatial variations in target clarity or information decay rates. The ergodic control using uniform TISD improves over lawnmower coverage by distributing effort more evenly; however, it still fails to prioritize regions according to their target clarity. In contrast, \genTISD{} dynamically allocates exploration effort toward regions with faster clarity decay or higher target clarity, resulting in more efficient information acquisition and lower overall clarity deficits over time.

It is important to note that in Environments 2 through 5, the mean clarity deficit does not converge to zero. This behavior is expected because the decay field $Q$ is non-zero in these environments, causing the target clarity to continuously degrade over time. As a result, it is not possible to achieve perfect target clarity even under optimal exploration. In contrast, Environment 1 has a decay field with $Q=0$ everywhere, allowing the robots to eventually drive the mean clarity deficit to zero through persistent exploration.
\begin{figure} [t]
    
  \centering
  \includegraphics[width=1.0\columnwidth]{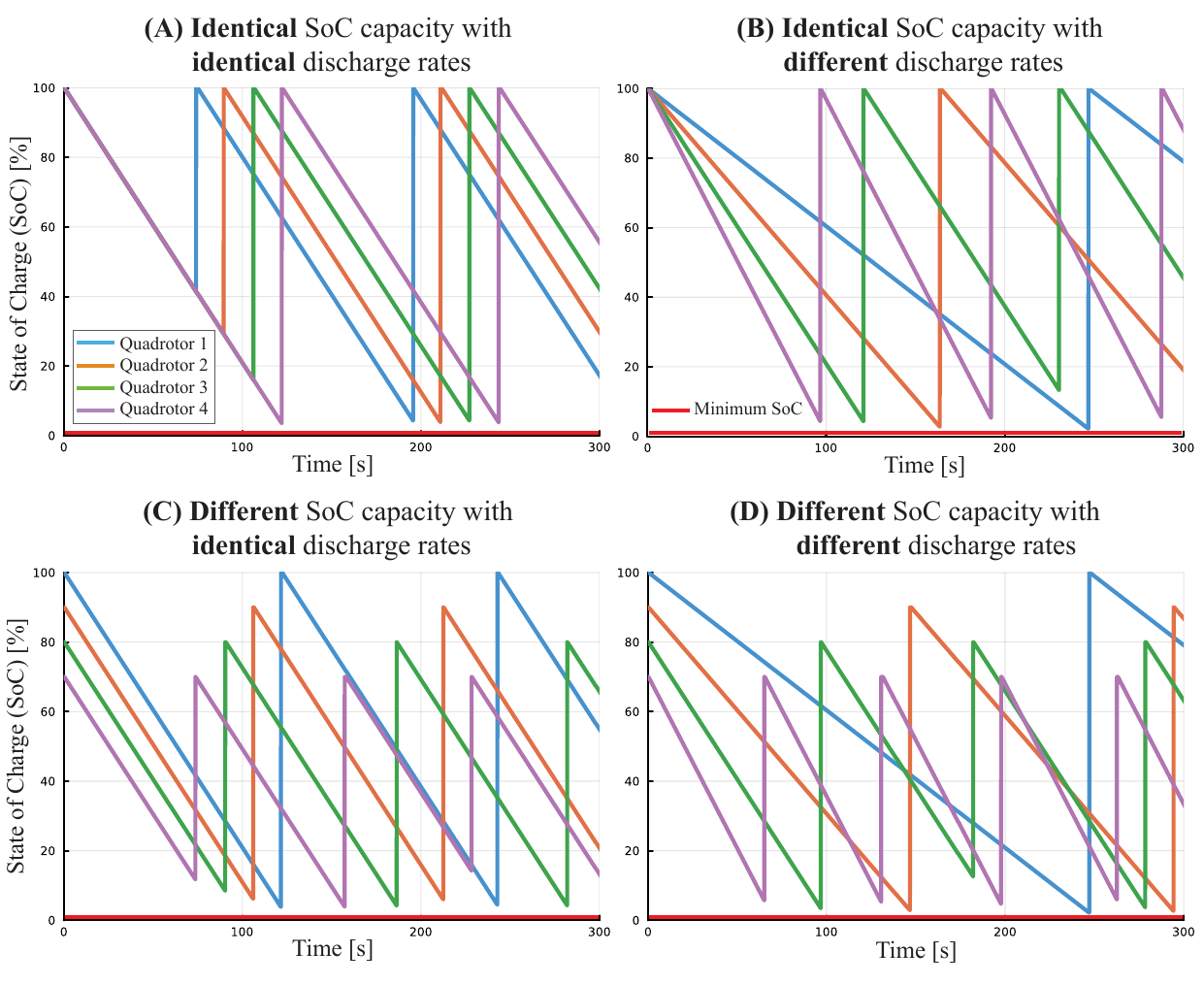}
  \caption{These plots show results for the scenarios when four quadrotors have different SoC capacities and different discharge rates. The plots validate that quadrotors always maintain the minimum of $(T_\delta)$ gap while visiting the charging station.}
  \label{fig:SoC_Dis_Plots}
\end{figure}

\subsection{\textbf{\Rmesch{} performance comparison to baseline methods}}

We compare \Rmesch{} to baseline methods using eight metrics, as shown in \cref{tab:meSch_comparison}. For each method, four robots are used with the same discharge model, $\Dot{e} = -0.667$. The total recharging visits are the same across all methods, except for Baseline 4, which focuses only on the timing of robot visits and does not account for the minimum energy requirements.

Compared to Baseline 1 (\cite{persis_Fouad}), \Rmesch{} supports nonlinear dynamic models, making it more applicable to real-world robotic platforms, as demonstrated with 3D quadrotor dynamics \cite{jackson2021planning}. Unlike Baseline 2 (\cite{bentz2018complete}), \mesch{} effectively handles both identical and varying discharge rates and state-of-charge (SoC) capacities without requiring robots to be deployed at different times. Deploying robots at different times reduces the number of robots available for the mission at any given moment, limiting overall efficiency. By allowing all robots to be deployed simultaneously, \Rmesch{} simplifies mission planning and increases adaptability to different discharge patterns, as shown in \cref{fig:SoC_Dis_Plots} with four quadrotors. Compared to Baseline 3 (\cite{naveed2023eclares}), \Rmesch{} eliminates simultaneous charging station visits. In Baseline 3, four robots returned concurrently on two occasions, leading to a violation of \eqref{eq:gap_lower_bound}. While Baseline 4, which only includes the \gware{} module from \mesch{} (\Rmesch{} without fail-safe planner), successfully avoids overlapping visits, it fails to enforce minimum energy constraints, resulting in a violation of \eqref{eq:min_energy_constraints}. Finally, none of the baseline methods support mobile charging stations—a limitation in environments where fixed charging locations may be infeasible. They also lack safe recovery maneuvers in the event of a central node failure. By addressing these gaps, \Rmesch{} enhances mission endurance and scalability while providing provable safety and feasibility guarantees.

\begin{figure} [t]
    
  \centering
  \includegraphics[width=1.0\columnwidth]{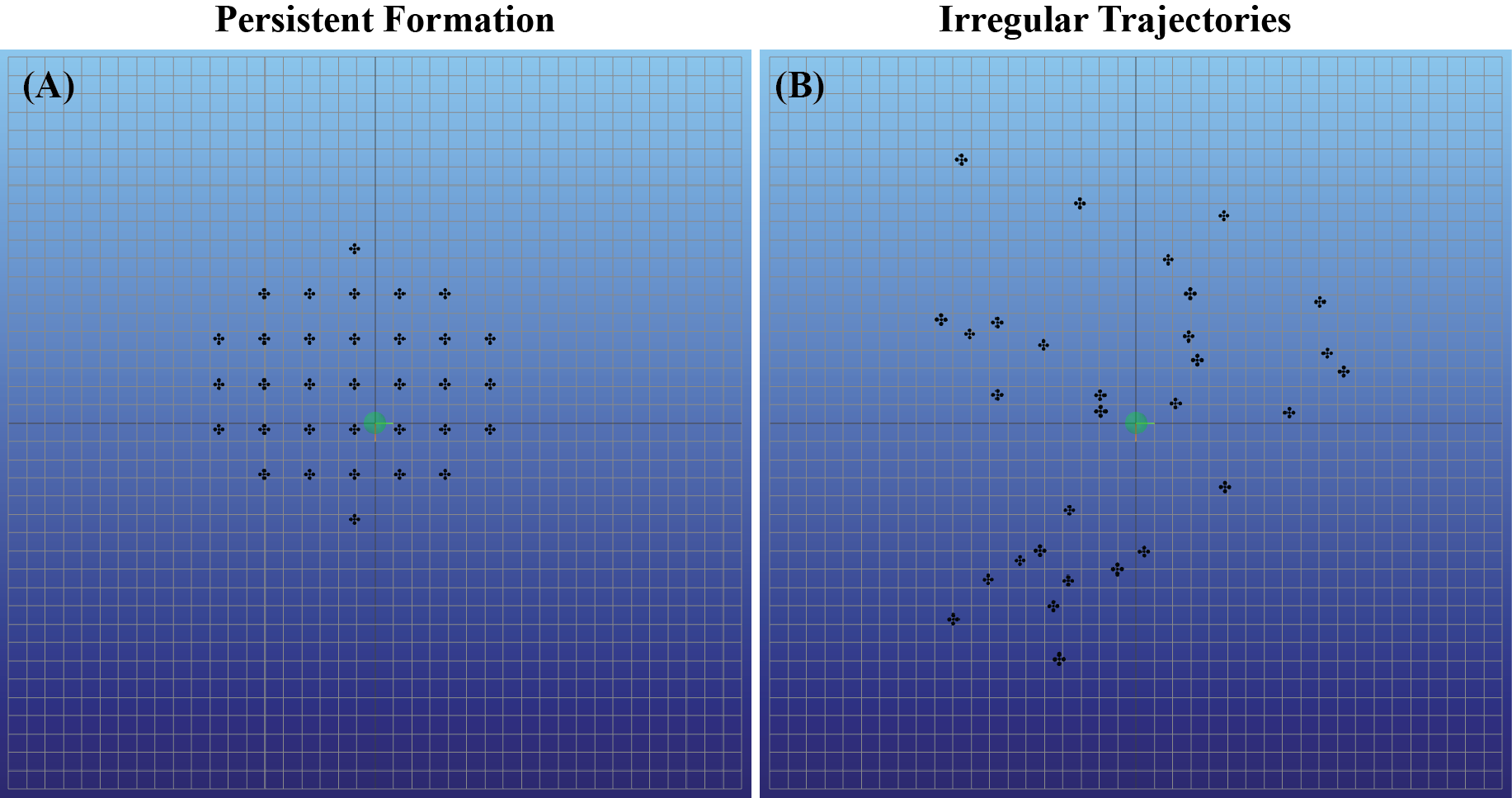}
  \caption{\Rmesch{} scalability is demonstrated through two simulation case studies: (A) 30 quadrotors performing a persistent mission, and (B) 30 quadrotors following irregular trajectories.}
  \label{fig:Scale_figure}
\end{figure}
 \begin{figure*}[t]
  \centering
  \includegraphics[width=2.1\columnwidth]{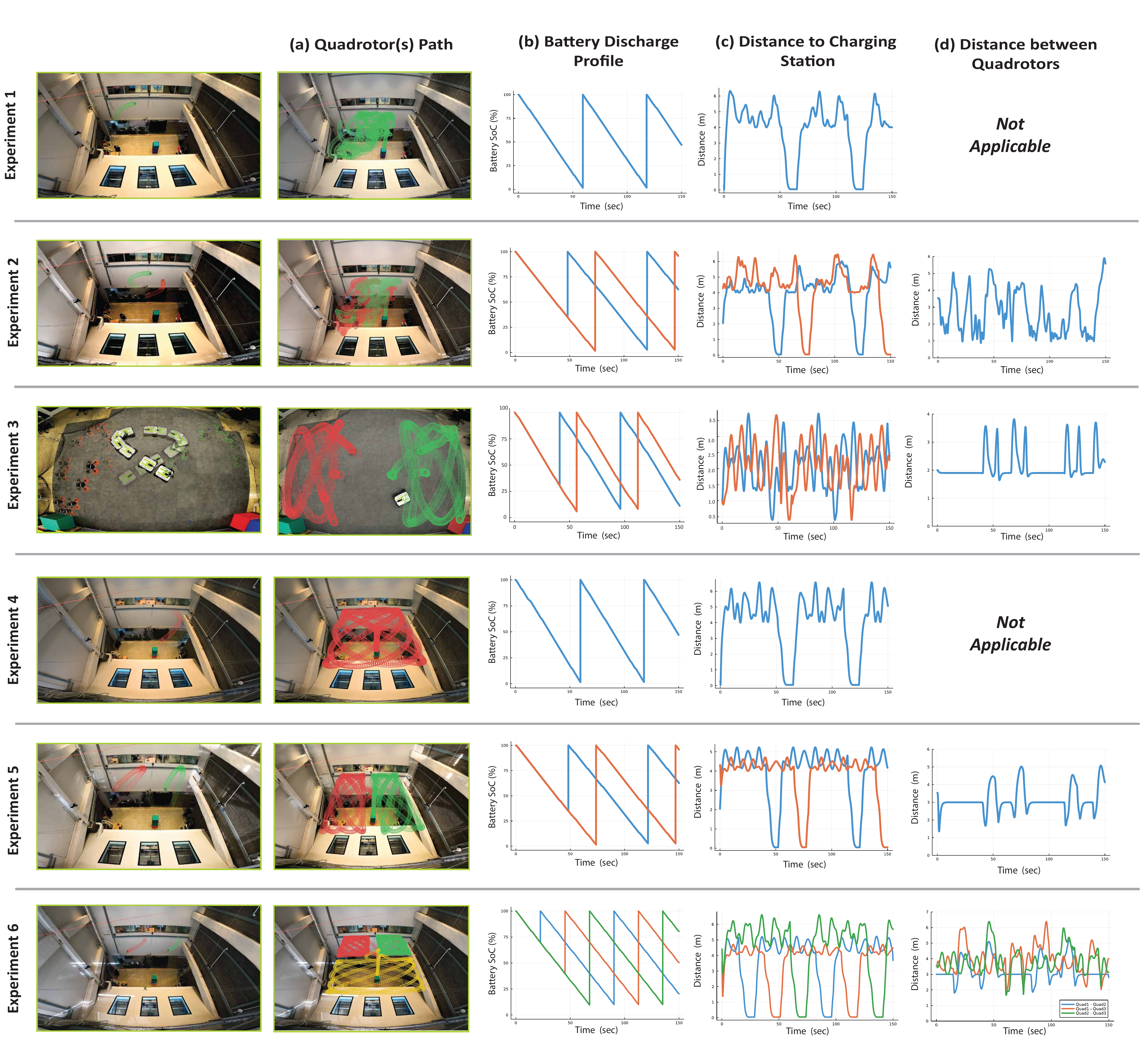}
  \caption{Demonstration of \mEclares{} and \Rmesch{} on hardware}
  \label{fig:Hardware_Experiments}
\end{figure*}
\subsubsection{\textbf{Computational efficiency and scalability}}

Distributing computation across the robot network improves the efficiency of the \Rmesch{} module. The main overhead comes from generating candidate trajectories, with solving \eqref{eq:b2b_prob_overall} and integrating the system’s nonlinear dynamics taking 150 ms and 30 ms on average, respectively. We employ the communication architecture(s) shown in
\cref{fig:comms}. 

To support real-time applications, each rechargeable robot (e.g. Quad 1) generates candidate trajectories on board, which are transmitted to the central node (Base) for scheduling. The scheduling algorithm has time complexity $\mathcal{O}(N \log N)$, mainly due to the sorting function in line 2 of \cref{alg:gware}. Thus, the method scales with $\mathcal{O}(N \log N)$, where $N$ is the number of rechargeable robots. To demonstrate scalability, we evaluate the method with 30 rechargeable quadrotors as shown in \cref{fig:Scale_figure}. In these simulations, quadrotors return with $(3\pm1) \%$ battery SoC remaining.

\subsubsection{\textbf{Hardware Demonstration}} 

We validate \mEclares{} through a set of real-world hardware experiments involving rechargeable quadrotors and a mobile charging rover. Each quadrotor runs onboard computation on an NVIDIA Orin NX, while the rover uses a Raspberry Pi. The communication architecture used in these experiments is shown in \cref{fig:comms}. All experiments were conducted in the FlyLab facility at Michigan Robotics—a three-floor indoor arena equipped with 15 Vicon cameras for high-precision state estimation.

In all experiments, only the next $T_N = 2.0$ seconds of the nominal trajectory is provided by the high-level planner. Candidate trajectories are generated onboard each quadrotor and transmitted to the base station computer, which verifies gap flags and minimum state-of-charge (SoC) conditions. The rover (when mobile charging is used) continuously publishes its own state and nominal trajectory to support trajectory generation. The experiments highlight three key aspects of our framework:
\begin{itemize}
\item the ability to generate ergodic trajectories online in real time,
\item the ability to generate candidate trajectories onboard each quadrotor at 1.0 Hz using only 2.0 seconds of the available nominal trajectory and validate them at a central node, and
\item the modularity of \Rmesch{}, which functions as a low-level scheduling module that remains effective even when the high-level planner is replaced with a non-ergodic coverage strategy.
\end{itemize}

Experiments are summarized in \cref{fig:Hardware_Experiments}. In the first set of experiments (Experiments 1–2), the quadrotors track ergodic trajectories that are replanned every 30 seconds. These trials validate that ergodic exploration and energy-aware scheduling can operate in tandem under real-world conditions. The target clarity for this set of experiments corresponds to Environment 2 in~\cref{fig:mClares_results}, where the quadrotors are observed to spend more time in regions with higher clarity deficit.

In Experiment 3, we demonstrate the use of a mobile charging station. We also show that the charging rover’s path can be changed to a Lissajous curve, and the framework still functions correctly—highlighting the flexibility of the \mEclares{} design.

Experiments 4–6 evaluate \Rmesch{} under a non-ergodic high-level planner. In these experiments, the quadrotors follow Lissajous coverage trajectories. Candidate trajectories are generated onboard every second and validated at the central node. \Rmesch{} continues to ensure safe and effective scheduling under this design.

Collision avoidance is implemented in all experiments using a potential field method, which generates artificial repulsive forces in real time to prevent inter-robot collisions.

\Cref{fig:Hardware_Experiments} (a) shows the coverage paths followed by the quadrotors. \cref{fig:Hardware_Experiments} (b) and (c) present the battery discharge profiles and distances to the charging station, respectively, confirming that robots never violate the minimum energy requirement (which is zero) and consistently satisfy the minimum desired gap requirement between charging returns. Finally, \cref{fig:Hardware_Experiments} (d) confirms that no collisions occur during the experiments.

 Our implementation also accounts for delays introduced by computational overhead and ROS2 message latency. The primary sources of delay include candidate trajectory generation and forward propagation ($T_1$), gap flag construction and verification ($T_2$), and message latency in ROS2 ($T_3$). As long as $T_1 + T_2 + T_3 < T_E$, where $T_E$ is the \Rmesch{} decision interval, the mission proceeds as intended. If these delays exceed the worst-case allowed duration, a fail-safe maneuver is triggered, prompting the quadrotors to return safely to the charging station. In our three-quadrotor experiments, we observed a latency of $600 \pm 150$ ms, with $T_E$ set to 1.5 s.

All simulation and experimental code is publicly released. \Rmesch{} is available as a Julia module that functions as a low-level filter for any high-level planner. We also provide a Python-ROS2 wrapper for Julia, a Docker container for easy deployment, and our in-house-developed Orin-based DevQuad platform~\cite{gatekeeper}.

\section{Conclusion}
This paper presented \mEclares{}, a unified framework for adaptive ergodic exploration and robust energy-aware scheduling in persistent multi-robot missions. We addressed two key challenges in long-term autonomous operations: (i) planning informative trajectories in stochastic spatiotemporal environments, and (ii) coordinating energy-constrained robots through a shared mobile charging station. By modeling information decay using the clarity metric and integrating it into ergodic search, we enabled the construction of time-evolving target information distributions that guide exploration under uncertainty. To ensure task persistence, we introduced \Rmesch{}, an online scheduling algorithm that guarantees mutually exclusive access to the charging station and provides robustness to communication delays and central node failures via fail-safe coordination.

Our approach supports general nonlinear dynamics, handles uncertain charging station state, and scales to teams of robots. Through extensive simulations and real-world hardware experiments, we demonstrated the effectiveness of \mEclares{} in maintaining persistent, informative coverage while adhering to energy and safety constraints. Theoretical guarantees further support the feasibility and robustness of our method under well-defined conditions.

Future work will explore extensions to fully decentralized scheduling under communication constraints, integration with online learning of environmental dynamics, and deployment in larger-scale, real-world missions with diverse robotic platforms.

\section{Acknowledgments}
The authors would like to acknowledge the support of the National Science Foundation (NSF) under grant no. 2223845 and grant no. 1942907. Moreover, the authors would like to acknowledge Haejoon Lee and Manveer Singh’s assistance with conducting the experiments.

\bibliographystyle{apacite}
\bibliography{biblio}

\begin{appendices}
\section{Proofs}\label{Appendix_proof}

\subsection{Proof of \Cref{lemma1}}\label{appen:lemma1_proof}

\begin{proof}
To ensure safe and sequential return of all $N^*$ robots, the algorithm requires that each return is separated by at least $T_\delta$ seconds. The robot with the smallest remaining flight time, $T_{F,0}^{1'}$, is assumed to return first. Each subsequent robot must return with a delay of at least $T_\delta$ from the previous one.

Therefore, the last robot (i.e., the $N^*$-th robot) must complete its return no later than
\eqn{
T_R + (N^* - 1) \cdot T_\delta + T_e,
}
where $T_R$ accounts for the time required to return, and $T_E$ accounts for a one-iteration delay before the return command can be issued.

To guarantee that even the last robot returns safely before depleting its energy, this total return time must be less than or equal to the smallest available remaining flight time:
\eqn{
T_R + (N^* - 1) \cdot T_\delta + T_E \leq T_{F,0}^{1'}.
}
Rearranging the inequality:
\eqn{
(N^* - 1) \cdot T_\delta \leq T_{F,0}^{1'} - T_R - T_E,
}
\eqn{
N^* - 1 \leq \frac{T_{F,0}^{1'} - T_R - T_E}{T_\delta},
}
\eqn{
N^* \leq 1 + \frac{T_{F,0}^{1'} - T_R - T_E}{T_\delta}.
}
Taking the floor on the right-hand side ensures conservativeness and integer feasibility:
\eqn{
N^* = 1 + \left\lfloor \frac{T_{F,0}^{1'} - T_R - T_E}{T_\delta} \right\rfloor.
}
\end{proof}

\subsection{Proof of \Cref{thm1}}
\label{appen:proof_theorem_1}
\begin{proof} We complete this proof considering two scenarios: In the first scenario, we prove recursive feasibility without central node failure. In the second scenario, we show that when the central node fails, the robots can be safely recovered while still respecting the constraints \eqref{eq:min_energy_constraints} and \eqref{eq:gap_lower_bound}.

\subsubsection{Feasibility guarantee without central node failure}

The proof, inspired by \cite[Thm. 1]{gatekeeper}, uses induction. 
\subsubsection*{Base Case} At the time $t_1$ and iteration $j = 1$, since both Gap flag condition \eqref{gap_flag} and the Reserve SoC condition \eqref{min_soc} are true, the candidate trajectories are committed for all rechargeable robots i.e. $\forall i \in \Rcal$ and $\forall t^{i_1}_{m_1},t^{i_2}_{m_2}  \in \Tcal$
\begin{subequations}
    \eqnN{
&x_{1}^{i, com}(t) \gets x_{1}^{i, can}(t)  \quad \forall t \in [t_1, t^i_{1,C})\\
&\implies \begin{cases}
    T_{F,1}^{k} > (T_R + T_{E} + kT_{\delta}) \quad \forall k \in \Rcal'\\
    e^i(t) > e^{res}_{1} \quad \forall t \in [t_{1}, t^i_{1, C})
\end{cases} \\
&\implies \begin{cases}
    |t^{i_1}_{m_1} - t^{i_2}_{m_2}| > T_{\delta}  &\forall t \in [t_{1}, t_{1, R} )   \\
    e^i(t) > e_{min}^i \quad &\forall t \in [t_{1}, t_{1, R} )
\end{cases} 
}
\end{subequations}

Since $t_{1, R} > t_{0, R} > t^i_{0, C} \forall i \in \Rcal$, the claim holds.
\subsubsection*{Induction Step} Suppose the claim is true for some $j \in \posintegers$. We show that the claim is true for $j + 1$. 
% The solution at the iteration $j+1$ has two definitions illustrated in \cref{fig:Proof_induc}:

\subsubsection*{Case 1} When candidate trajectories for all rechargeable robots are valid, i.e. $\forall i \in \Rcal$ and and $\forall t^{i_1}_{m_1},t^{i_2}_{m_2}  \in \Tcal$
\begin{subequations}
    \eqnN{
&x_{j+1}^{i, com}(t) \gets x_{j+1}^{i, can}(t) \quad \forall t \in [t_{j+1}, t^i_{j+1,C})\\
&\implies \begin{cases}
    T_{F,j+1}^{k} > (T_R + T_{E} + kT_{\delta}) \quad \forall k \in \Rcal'\\
    e^i(t) > e^{res}_{j+1} \quad \forall t \in [t_{j+1}, t^i_{j+1, C})
\end{cases} \\
&\implies \begin{cases}
    |t^{i_1}_{m_1} - t^{i_2}_{m_2}| > T_{\delta}  &\forall t \in [t_{j+1}, t_{j+1, R} )  \\
    e^i(t) > e_{min}^i \quad &\forall t \in [t_{j+1}, t_{j+1, R} )
\end{cases} 
}
\end{subequations}
Since $ t_{j+1, R} >  t_{j, R}, \forall i \in \Rcal$  the claim holds.
\subsubsection*{Case 2} This case corresponds to the scenario when the $1^{'th}$ robot in $\Rcal'$ returns either due to violation of Gap flag condition or the Reserve SoC condition, i.e.,
\eqnN{
x_{j+1}^{1', com}(t) \gets x_{j}^{1', com}(t) \quad \forall t \in [t_{j+1}, t^{1'}_{j,C}).}
The candidate trajectories are committed for the remaining robots, i.e. $\forall k \in \Rcal^'\backslash\{1'\}$ and $\forall t^{i_1}_{m_1},t^{i_2}_{m_2}  \in \Tcal$
\begin{subequations}
    \eqnN{
&x_{j+1}^{k, com} \gets x_{j+1}^{k, can} \quad  \forall t \in [t_{j+1}, t^k_{j+1,C})\\
&\implies \begin{cases}
    T_{F,j+1}^{k} > (T_R + T_{E} + kT_{\delta}) \\
    e^k(t) > e^{res}_{j+1} \quad \forall t \in [t_{j+1}, t^k_{j+1, C})
\end{cases} \\
&\implies \begin{cases}
    |t^{i_1}_{m_1} - t^{i_2}_{m_2}| > T_{\delta}  &\forall t \in [t_{j+1}, t_{j+1, R} )  \\
    e^k(t) > e_{min}^k \quad &\forall t \in [t_{j+1}, t_{j+1, R} )
\end{cases} 
}
\end{subequations}

Since $t_{j+1, R} > t^k_{j, C}$,  the claim holds.

\subsubsection{Safe recovery and feasibility guarantee with central node failure}

Suppose that at time $t_j$, all rechargeable robots generate new candidate trajectories and send requests to the central node for validation. Each robot then waits for a response until $t_{j-1,N}$, as specified by the fail-safe protocol. If no message is received from the central node by this deadline, each robot executes its onboard fail-safe maneuver using its previously committed trajectory $x_{j-1}^{i,\text{com}}$ and its stored return index $\text{ret}_{j-1}^i$.

\begin{itemize}
    \item If $\text{ret}_{j-1}^i = 1$, the robot immediately continues following $x_{j-1}^{i,\text{com}}$, ensuring a return by $t_j + T_R$.
    \item If $\text{ret}_{j-1}^i > 1$, the robot idles for $\text{ret}_{j-1}^i T_\delta$ seconds and then executes a time-shifted version of $x_{j-1}^{i,\text{com}}$ over the interval $[t_j + \text{ret}_{j-1}^i T_\delta,\, t_j + \text{ret}_{j-1}^i T_\delta + T_R]$.
\end{itemize}

This structure ensures two properties:
\begin{enumerate}
    \item \textbf{Gap constraint satisfaction:} The time-shifting mechanism guarantees that no two robots attempt to return simultaneously. Since each robot delays its return by $(\text{ret}_{j-1}^i - 1) \cdot T_\delta$, mutual exclusion at the charging station is preserved, and the gap condition \eqref{eq:gap_lower_bound} holds.
    \item \textbf{Minimum energy constraint satisfaction:} Since each robot had already committed a feasible trajectory at $t_{j-1}$ with enough energy to return after the assigned delay, the energy constraint \eqref{eq:min_energy_constraints} remains satisfied.
\end{enumerate}

Thus, even in the absence of centralized coordination, the robots return safely, respecting both the return gap and minimum energy requirements. Hence, recursive feasibility also holds under central node failure.
\end{proof}

\end{appendices}

\end{document}